%% file: main.tex
\documentclass{article}

\usepackage{microtype}
\usepackage{graphicx}
\usepackage{subfigure}
\usepackage{booktabs} % for professional tables
\usepackage{amsfonts}       % blackboard math symbols
\usepackage{nicefrac}       % compact symbols for 1/2, etc.
\usepackage{amsmath}
\usepackage{float}
\usepackage{lipsum}
\usepackage{color, xcolor, colortbl, hhline}
\newcolumntype{C}[1]{>{\centering\arraybackslash}m{#1}}
\usepackage{tabularx,ragged2e,booktabs,caption}

\usepackage{hyperref}
\hypersetup{
    colorlinks=true,
    linkcolor=blue,
    filecolor=magenta,      
    urlcolor=cyan,
}

\usepackage{amsthm}
\usepackage{amssymb}
\usepackage{makecell}
\usepackage{bbm}
\usepackage{times}
\usepackage{caption}

\newtheorem{innercustomgeneric}{\customgenericname}
\providecommand{\customgenericname}{}
\newcommand{\newcustomtheorem}[2]{%
  \newenvironment{#1}[1]
  {%
   \renewcommand\customgenericname{#2}%
   \renewcommand\theinnercustomgeneric{##1}%
   \innercustomgeneric
  }
  {\endinnercustomgeneric}
}

\newcustomtheorem{customthm}{Theorem}
\newcustomtheorem{customlemma}{Lemma}
\newcustomtheorem{customalg}{Algorithm}
\newcustomtheorem{customproposition}{Proposition}

\usepackage[accepted]{icml2021}

\icmltitlerunning{Locally Persistent Exploration in Continuous Control Tasks with Sparse Rewards}

\input{macros}

\begin{document}

\twocolumn[
\icmltitle{Locally Persistent Exploration in Continuous Control Tasks with Sparse Rewards}



\icmlsetsymbol{equal}{*}

\begin{icmlauthorlist}
\icmlauthor{Susan Amin}{equal,McGill,Mila}
\icmlauthor{Maziar Gomrokchi}{equal,McGill,Mila}
\icmlauthor{Hossein Aboutalebi}{equal,Waterloo,art}
\icmlauthor{Harsh Sajita}{McGill,Mila}
\icmlauthor{Doina Precup}{McGill,Mila}
\end{icmlauthorlist}

\icmlaffiliation{McGill}{Department of Computer Science, McGill University, Montréal, Québec, Canada}
\icmlaffiliation{Mila}{Mila - Québec Artificial Intelligence Institute, Montréal, Québec, Canada}
\icmlaffiliation{Waterloo}{Department of Computer Science, University of Waterloo, Waterloo, Ontario, Canada}
\icmlaffiliation{art}{Waterloo Artificial Intelligence Institute, University of Waterloo, Waterloo, Ontario, Canada}

\icmlcorrespondingauthor{Susan Amin}{susan.amin@mail.mcgill.ca}

\icmlkeywords{Machine Learning, ICML}

\vskip 0.3in
]



\printAffiliationsAndNotice{\icmlEqualContribution} % otherwise use the standard text.

\begin{abstract}
A major challenge in reinforcement learning is the design of exploration strategies, especially for environments with sparse reward structures and continuous state and action spaces. Intuitively, if the reinforcement signal is very scarce, the agent should rely on some form of short-term memory in order to cover its environment efficiently. We propose a new exploration method, based on two intuitions: (1) the choice of the next exploratory action should depend not only on the (Markovian) state of the environment, but also on the agent's trajectory so far, and (2) the agent should utilize a measure of spread in the state space to avoid getting stuck in a small region. Our method leverages concepts often used in statistical physics to provide explanations for the behavior of simplified (polymer) chains in order to generate persistent (locally self-avoiding) trajectories in state space. We discuss the theoretical properties of locally self-avoiding walks and their ability to provide a kind of short-term memory through a decaying temporal correlation within the trajectory. We provide empirical evaluations of our approach in a simulated 2D navigation task, as well as higher-dimensional MuJoCo continuous control locomotion tasks with sparse rewards.
\end{abstract}

\section{Introduction}
As reinforcement learning agents typically learn tasks through interacting with the environment and receiving reinforcement signals, a fundamental problem arises when these signals are rarely available. The sparsely distributed rewards call for a clever exploration strategy that exposes the agent to the unseen regions of the space via keeping track of the visited state-action pairs \cite{fu2017ex2,nair2018overcoming}. However, that cannot be the case for high-dimensional continuous space-and-action spaces, as defining a notion of density for such tasks is intractable and heavily task-dependent \cite{NIPS2017_7090, taiga2019benchmarking}.

Here, we introduce an exploration algorithm that works independently of the extrinsic rewards received from the environment and is inherently compatible with continuous state-and-action tasks. Our proposed approach takes into account the agent's short-term memory regarding the action trajectory, as well as the trajectory of the observed states in order to sample the next exploratory action. The main intuition is that in a pure exploration mode with minimal extrinsic reinforcement, the agent should plan trajectories that expand in the available space and avoid getting stuck in small regions. In other words, the agent may need to be ``persistent" in its choice of actions; for example, in a locomotion task, an agent may want to pick a certain direction and maintain it for some number of steps in order to ensure that it can move away from its current location, where it might be stuck at. The second intuition is that satisfying the first condition requires a notion of spread measure in the state space to warrant the agent's exposure to unvisited regions. Moreover, in sparse reward settings, while the agent's primary intention must be to avoid being trapped in local regions by maintaining a form of short-term memory, it must still employ a form of memory evaporation mechanism to maintain the possibility of revisiting the informative states. Note that in continuous state-and-action settings, modern exploration methods \cite{ostrovski2017count, houthooft2016vime, ciosek2019better} fail to address the fore-mentioned details simultaneously.

Our polymer-based exploration technique (PolyRL) is inspired by the theory of freely-rotating chains (FRCs) in polymer physics to implement the aforementioned intuitions. FRCs describe the chains (collections of transitions or moves) whose successive segments are correlated in their orientation. This feature introduces a finite (short-term) stiffness (persistence) in the chain, which induces what we call {\em locally} self-avoiding random walks (LSA-RWs). The strategy that emerges from PolyRL induces consistent movement, without the need for exact action repeats (\emph{e.g.} methods suggested by \cite{dabney2020temporally, lakshminarayanan2017dynamic, sharma2017learning}), and can maintain the rigidity of the chain as required. Moreover, unlike action-repeat strategies, PolyRL is inherently applicable in continuous action-state spaces without the need to use any discrete representation of action or state space. The local self-avoidance property in a PolyRL trajectory cultivates an orientationally persistent move in the space while maintaining the possibility of revisiting different regions in space. 

To construct LSA-RWs, PolyRL selects persistent actions in the action space and utilizes a measure of spread in the state space, called the \emph{radius of gyration}, to maintain the (orientational) persistence in the chain of visited states. The PolyRL agent breaks the chain and performs greedy action selection once the correlation between the visited states breaks. The next few exploratory actions that follow afterward, in fact, act as a perturbation to the last greedy action, which consequently preserves the orientation of the greedy action. This feature becomes explicitly influential after the agent is exposed to some reinforcement, and the policy is updated, as the greedy action guides the agent's movement through the succeeding exploratory chain of actions.

\vspace{-0.1in}
\section{Notation and Problem Formulation}\label{sub:Notation}
We consider the usual MDP setting, in which an agent interacts with a Markov Decision Process $\langle \mathcal{S}, \mathcal{A}, P, r\rangle$, where
$\mathcal{S}\subseteq \mathbb{R}^{d_\mathcal{S}}$ and $\mathcal{A}\subseteq\mathbb{R}^{d_\actions}$ are continuous state and action spaces, respectively; $P: \mathcal{S}\times\mathcal{A} \rightarrow (\mathcal{S}\rightarrow[0,1])$ represents the transition probability kernel, and $r:\mathcal{S} \times \mathcal{A} \rightarrow \mathbb{R}$ is the reward function. Moreover, we make a smoothness assumption on $P$,
\begin{assumption}\label{as:1}
The transition probability kernel $P$ is Lipschitz w.r.t. its action variable, in the sense that there exists $C > 0$ such that for all $(s,a,a') \in \states \times \actions \times \actions$ and measurable set $\mathcal{B} \subset \states$,
\begin{align}
    |P(\mathcal{B}|s,a) - P(\mathcal{B}|s,a')| \leq C \norm{a - a'}.
\end{align}
\end{assumption}

Assumption \ref{as:1} has been used in the literature for learning in domains with continuous state-action spaces \cite{antos2008fitted}, as the assumption on the smoothness of MDP becomes crucial in such environments \cite{antos2008fitted,bartlett2007sample,ng2000pegasus}. Note that we only use the Lipschitz smoothness assumption on $P$ for the theoretical convenience, and we later provide experimental results in environments that do not satisfy this assumption. Furthermore, we assume that the state and action spaces are inner product spaces. 

We show the trajectory of selected exploratory actions in the action space $\actions$ as $\tau_\actions = (a_0,...,a_{T_e-1})$ and the trajectory of the visited states in the state space $\states$ as $\tau_\states = (s_0,...,s_{T_e-1})$, where $T_e$ denotes the number of exploratory time steps in a piece-wise persistent trajectory, and is reset after each exploitation step. Moreover, we define
\begin{align}
\Omega(\tau_\states, \tau_\actions):= \{s \in \states | \Pr[S_{T_e}=s | s_{T_e-1}, a_{T_e-1}] > 0\}\label{eq:omega}
\end{align}
as the set of probable states $s\in\states$ observed at time $T_e$ given $s_{T_e-1}$ from $\tau_\states$ and the selected action $a_{T_e-1}$ from $\tau_\actions$. For simplicity, in the rest of this manuscript, we denote $\Omega(\tau_\states, \tau_\actions)$ by $\Omega$. In addition, the concatenation of the trajectory $\tau_\states$ and the state $s$ visited at time step $T_e$ is denoted as the trajectory $\tau'_\states := (\tau_\states,s)$. For the theoretical analysis purposes, and in order to show the expansion of the visited-states trajectory in the state space, we choose to transform $\tau_\states$ into a sequence of vectors connecting every two consecutive states (bond vectors),
\begin{align}\label{eq:bond-vector}
\omega_{\tau_\states} =& \{\omega_i\}_{i=1}^{T_e-1}, ~~~where~~~ \omega_i= s_i-s_{i-1}.
\end{align}
Finally, we define Self-Avoiding Random Walks (LSA-RWs), inspired by the theory of freely-rotating chains (FRCs), where the correlation between $|i-j|$ consecutive bond vectors is a decaying exponential with the correlation number $Lp$ (persistence number). $Lp$ represents the number of time steps, after which the bond vectors forget their initial orientation.

\begin{definition}\label{def:LSAW} 
{\normalfont\textbf{Locally Self-Avoiding Random Walks (LSA-RWs)}}
A sequence of random bond vectors $\omega = \{w_i\}_{i=1}^{T_e}$ is Locally Self-Avoiding with persistence number $Lp>1$, if there exists $b_o>0$ for all $i,j \in [T_e]$ such that, 1) $\mathbb{E}[\norm{w_i}^2] = b_o^2$ and 2) $\mathbb{E}[ w_i.w_j]\sim b^2_o e^{\frac{-|j-i|}{Lp}}$, where $\mathbb{E}[.]$ is computed over all configurations of the chain induced by the dynamic randomness of the environment (equivalent to the thermal fluctuations in the environment in statistical physics).
\end{definition}

The first condition states that the expected magnitude of each bond vector is $b_o$. The second condition shows that the correlation between $|i-j|$ consecutive bond vectors is a decaying exponential with the correlation length $Lp$ (persistence number), which is the number of time steps after which the bond vectors forget their original orientation. Note that despite the redundancy of the first condition, we choose to include it separately in order to emphasize the finite expected magnitude of the bond vectors.
\begin{figure*}[h!]
\begin{center}
\includegraphics[width=1 \textwidth]{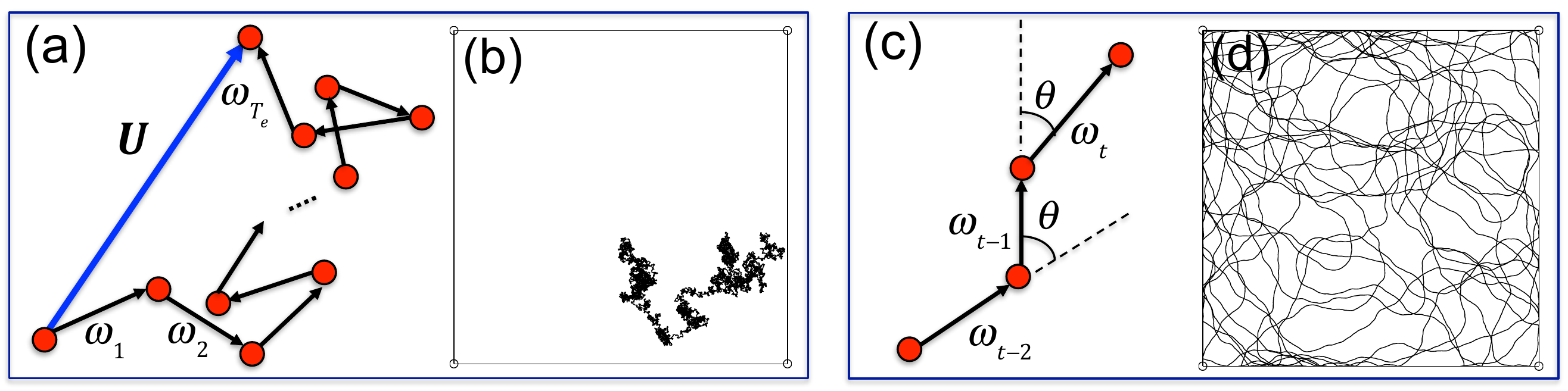}
\end{center}
\caption{ A chain (or trajectory) is shown as a sequence of $T_e$ random bond vectors $\{\boldsymbol{\omega}_i\}_{i=1..T_e}$. In a freely-jointed chain (FJC) (a), the orientation of the bond vectors are independent of one another (Random Walk). In a freely-rotating chain (FRC) (c), the correlation angle $\theta$ is invariant between every two consecutive bond vectors, which induces a finite stiffness in the chain. (b, d) A qualitative comparison between an FJC (b) and an FRC with $\theta\approx5.7^\circ$ (d), in a 2D environment of size $400\times400$ for $20000$ number of steps.}
\label{fig:Polymer-def}
\end{figure*}

\vspace{-0.1in}
\section{Methods}\label{sec:methods}
We introduce the method of polymer-based exploration in reinforcement learning (PolyRL), which borrows concepts from Statistical Physics \cite{de1979scaling, doi1988theory} to induce persistent trajectories in continuous state-action spaces (Refer to Figure \ref{fig:Polymer-def} for a schematic of simplified polymer models and the Appendix (Section \ref{sec:poly-models}) for more information regarding polymer models). As discussed below, our proposed technique balances exploration and exploitation using high-probability confidence bounds on a measure of spread in the state space. Algorithm \ref{alg:plyrl} presents the PolyRL pseudo code. The method of action sampling is provided in Algorithm \ref{alg:actionSampling} in the Appendix (Section \ref{sec:action-sampling}).

The PolyRL agent chooses the sequence of actions in $\actions$ such that every two consecutive action vectors are restricted in their orientation with the mean [correlation] angle $\theta$. In order to induce persistent trajectories in the state space, the agent uses a measure of spread in the visited states in order to ensure the desired expansion of the trajectory $\tau_\states$ in $\states$. We define \emph{radius of gyration squared},
\begin{align}
    U_g^2(\tau_\states) &:= \frac{1}{T_e-1} \sum_{s \in \tau_\states} d^2(s,\Bar{\tau}_\states),\label{eq:radius_gyration}
\end{align}
as a measure of the spread of the visited states in the state space $\states$, where $d(s,\bar{\tau}_\states)$ is a metric defined on the state space $\states$, and serves as a measure of distance between the visited state $s$ and the empirical mean of all visited states $\bar{\tau}_\states$. Also known as the center of mass, $\bar{\tau}_\states$ is calculated as, $\Bar{\tau}_{\states}:=\frac{1}{T_e}\sum_{s \in \tau_\states} s$.

\begin{algorithm}[h!]
\caption{PolyRL Algorithm}
\label{alg:plyrl}
\begin{algorithmic}
\REQUIRE{Exploration factor $\beta$, Average correlation angle $\theta$ and Correlation variance $\sigma^2$}
\FOR {N in total number of episodes}
	\STATE $\delta \leftarrow 1- e^{-\beta N}$ \COMMENT{An increasing function of the the episode number.}
	\STATE Sample ${\bf a}_0$ and ${\bf s}_0$ from an initial 	state-action distribution, and Exploit-flag $\leftarrow 0$ 
	\COMMENT{If Exploit-flag is $1$, the agent uses the target policy to select action, and explores otherwise.}
	\REPEAT
	    \IF{Exploit-flag == 1}
	        \STATE Draw a random number $\kappa\sim\mathcal{N}\left(0,1\right)$
	        \IF{$\kappa\leq\delta$}
	            \STATE ${\bf a}_t\sim \pi_\mu$ \COMMENT{action is drawn from the target policy}
	        \ELSE
	            \STATE Exploit-flag $\leftarrow 0$
	            \STATE Start a new exploratory trajectory by setting the radius of gyration squared to zero
	            \STATE Draw random number $\eta \sim \mathcal{N}( \theta, \sigma^2)$
		    	\STATE ${\bf a}_t\sim \pi_{\mbox{\scriptsize PolyRL}}(\eta, \bf a_{t-1})$
		    \ENDIF
	    \ELSE
		    \STATE Compute the change in the radius of gyration squared $\Delta U_g^2$ letting $d=L_2$-norm and
		    \STATE using eq. \eqref{eq:radius_gyration}.
		    \STATE Calculate $UB$ and $LB$ using eqs. \eqref{eq:upper} and \eqref{eq:lower}.
		    \IF{$\Delta U_g^2\geq LB$ and $\Delta U_g^2\leq UB$}
		   	    \STATE Draw random number $\eta \sim \mathcal{N}( \theta, \sigma^2)$
		   	    \STATE ${\bf a}_t\sim \pi_{\mbox{\scriptsize PolyRL}}(\eta, \bf a_{t-1})$
		    \ELSE
		        \STATE Exploit-flag $\leftarrow 1$, and ${\bf a}_t\sim \pi_\mu$
            \ENDIF
		%	\State Start a new trajectory
		\ENDIF
	\UNTIL{the agent reaches an absorbing state or the maximum allowed time steps in an episode}
\ENDFOR
\end{algorithmic}
\end{algorithm}

At each time step, the agent calculates the radius of gyration squared (Equation \ref{eq:radius_gyration}) and compares it with the obtained value from the previous time step. If the induced trajectory in the state space is LSA-RW, it maintains an expected stiffness described by a bounded change in $U_g^2$. Refer to Theorems \ref{theory:upper} and \ref{theory:lower} in Section \ref{sec:theory} for detailed information regarding the calculation of the bounds. In brief, the two theorems present high-probability confidence bounds on upper local sensitivity $UB$ and lower local sensitivity $LB$, respectively. The lower bound ensures that the chain remains LSA-RW and expands in the space, while the upper bound prevents the agent from moving abruptly in the environment (The advantage of using LSA-RWs to explore the environment can be explained in terms of their high expansion rate, which is presented in Proposition \ref{prop:expanding} in the Appendix (Section \ref{sec:poly-models})). If the computed $\Delta U_g^2$ is in the range $[LB,UB]$, the agent continues to perform PolyRL action sampling method (Algorithm \ref{alg:actionSampling} in the Appendix). Otherwise, it samples the next action using the target policy. Due to the persistence of the PolyRL chain of exploratory actions, the orientation of the last greedy action is preserved for $Lp$ (persistence) number of steps. As for the trajectories in the state space, upon observing correlation angles above $\pi/2$, the exploratory trajectory is broken and the next action is chosen with respect to the target policy $\pi_\mu$. 

\vspace{-0.1in}
\section{Theory}\label{sec:theory}
In this section, we derive the upper and lower confidence bounds on the local sensitivity for the radius of gyration squared between $\tau_\states$ and $\tau^{\prime}_\states$ (All proofs are provided in the Appendix (Section \ref{sec:proofs})). Given the trajectory $\tau_\states$ and the corresponding radius of gyration squared $U_g^2(\tau_\states)$ and persistence number $Lp_{\tau_\states}>1$, we seek a description for the permissible range of $U_g^2(\tau^{\prime}_\states)$ such that the stiffness of the trajectory is preserved. Note that the derived equations for the upper and lower confidence bounds are employed in the PolyRL algorithm (Algorithm \ref{alg:plyrl}) in bounding the change in $U_g^2$ to ensure the formation of locally self-avoiding walks in the trajectory of visited states in the state space. 

\textbf{High-probability upper bound -} We define the upper local sensitivity on $U_g^2$ upon observing new state $s_{T_e} \in \states$ as,
\begin{align}
    ULS_{Ug^2}(\tau_\states) := \sup_{s_{T_e} \in \Omega} U_g^2(\tau'_\states) - U_g^2(\tau_\states).\label{eq:upper_bound_def}
\end{align}
Given the observed state trajectory $\tau_\states$ with persistence number $Lp_{\tau_\states}$, the upper local sensitivity $ULS_{Ug^2}$ provides the maximum change observed upon visiting the next accessible state $s^\prime\in\Omega$. With the goal of constructing the new trajectory $\tau'_\states$ such that it preserves the stiffness induced by $\tau_\states$, we calculate the high-probability upper confidence bound on $ULS_{Ug^2}$. To do so, we write the term $d^2(s,\bar{\tau}_\states)$ in Equation \ref{eq:radius_gyration} as a function of bond vectors $\omega_i$, which is presented in Lemma \ref{lem:bond-transform} given that $d= L_2$-norm. We further substitute the resulting  $U_g^2(\tau_\states)$ in Equation \ref{eq:upper_bound_def} with the obtained expression from Equation \ref{eq:radius_gyration}.

\begin{lemma}\label{lem:bond-transform}
Let $\tau_\states = (s_0, \dots, s_{T_e-1})$ be the trajectory of visited states, $s_{T_e}$ be a newly visited state and $\boldsymbol{\omega}_i = s_i-s_{i-1}$ be the bond vector that connects two consecutive visited states $s_{i-1}$ and $s_{i}$. Then we have,
 \begin{align}
      \norm{s_{T_e}- \bar{\tau}_\states}^2 = \norm{\boldsymbol{\omega}_{T_e} + \frac{1}{T_e} \left[ \sum_{i=1}^{T_e-1} i \boldsymbol{\omega}_i \right]}^2.\label{eq:lem-bond-transform}
 \end{align}
\end{lemma}
The result of Lemma \ref{lem:bond-transform} will be used in the proof of Theorem \ref{theory:upper}, as shown in the Appendix (Section \ref{sec:proofs}). In Theorem \ref{theory:upper}, we provide a high-probability upper bound on $ULS_{Ug^2}(\tau_\states)$. 

\begin{theorem}[Upper-Bound Theorem]\label{theory:upper}
Let $\delta \in (0,1)$ $\tau_\states$ be an LSA-RW in $\states$ induced by PolyRL with the persistence number $Lp_{\tau_\states}>1$ within episode $N$, $\omega_{\tau_\states}=\{\boldsymbol{\omega}_i\}_{i=1}^{T_e-1}$ be the sequence of corresponding bond vectors, where $T_e>0$ denotes the number of bond vectors within $\tau_\states$, and $b_o$ be the average bond length. The upper confidence bound for $ULS_{Ug^2}(\tau_\states)$ with probability of at least $1-\delta$ is, 

\begin{align}\label{eq:upper}
  UB = &\Lambda (T_e,\tau_\states)\\ \nonumber +&\frac{1}{\delta}{\left[ \Gamma (T_e, b_o, \tau_\states) + \frac{2b^2_o}{T_e^2} \sum_{i=1}^{T_e-1} i  e^{\frac{-(T_e-i)}{Lp_{\tau_\states}}} \right]},
\end{align}
where,
\begin{align}
    \Lambda (T_e,\tau_\states) =& -\frac{1}{T_e-1} U_g^2(\tau_\states) \\
    \Gamma (T_e, b_o, \tau_\states) =& \frac{b^2_o}{T_e}  + \frac{\norm{\sum_{i=1}^{T_e-1} i\boldsymbol{\omega}_i}^2}{T_e^3}
\end{align}
\end{theorem}
Equation \ref{eq:upper} provides an upper bound on the pace of the trajectory expansion in the state space (denoted by $\Delta U_g^2\leq UB$ in Algorithm \ref{alg:plyrl}) to prevent the agent from moving abruptly in the environment, which would otherwise lead to a break in its temporal correlation with the preceding states. Similarly, we introduce the lower local sensitivity $LLS_{Ug^2}$, which provides the minimum change observed upon visiting the next accessible state $s^\prime\in\Omega$.

\textbf{High-probability lower bound -} In this part, We define the lower local sensitivity on $U_g^2$ upon observing new state $s_{T_e} \in \states$ as,
\begin{align}\label{eq:lower_bound_def}
     LLS_{Ug^2}(\tau_\states) := \inf_{s_{T_e} \in \Omega} U_g^2(\tau'_\states) - U_g^2(\tau_\states).
\end{align}
 We further compute the high-probability lower confidence bound on $LLS_{Ug^2}$ in order to guarantee the expansion of the trajectory $\tau_\states$ upon visiting the next state.

\begin{theorem}[Lower-Bound Theorem]\label{theory:lower}
Let $\delta \in (0,1)$ and $\tau_\states$ be an LSA-RW in $\states$ induced by PolyRL with the persistence number $Lp_{\tau_\states}>1$ within episode $N$, $\omega_{\tau_\states}=\{\boldsymbol{\omega}_i\}_{i=1}^{T_e-1}$ be the sequence of corresponding bond vectors, where $T_e>0$ denotes the number of bond vectors within $\tau_\states$, and $b_o$ be the average bond length. The lower confidence bound on $LLS_{Ug^2}(\tau_\states)$ at least with probability $1-\delta$ is, 
\small
\begin{align}\label{eq:lower}
    LB =& \Lambda (T_e,\tau_\states)
    +(1-\sqrt{2-2\delta})\times\nonumber\\
    &\left[\Gamma (T_e, b_o, \tau_\states) +\frac{(T_e-1)(T_e-2)}{T_e^2} b^2_0 e^{\frac{-|T_e-1|}{Lp_{\tau_\states}}}   \right],
\end{align}
where,
\begin{align}
    \Lambda (T_e,\tau_\states) =& -\frac{1}{T_e-1} U_g^2(\tau_\states) \\
    \Gamma (T_e, b_o, \tau_\states) =& \frac{b^2_o}{T_e}  + \frac{\norm{\sum_{i=1}^{T_e-1} i\boldsymbol{\omega}_i}^2}{T_e^3}
\end{align}
\normalsize
\end{theorem}
Equation \ref{eq:lower} provides a lower bound on the change in the expansion of the trajectory (\emph{e.g.} $\Delta U_g^2\geq LB$ as shown in Algorithm \ref{alg:plyrl}). 

The factor $\delta\in[0,1]$ that arises in Equations \ref{eq:upper} and \ref{eq:lower} controls the tightness of the confidence interval. In order to balance the trade-off between exploration and exploitation, we choose to increase $\delta$ with time, as increasing $\delta$ leads to tighter bounds and thus, higher probability of exploitation. In addition, the exploration factor $\delta$ determines the probability of switching from exploitation back to starting a new exploratory trajectory. Note that for experimental purposes, we let $\delta = 1 - e^{-\beta N}$, where the exploration factor $\beta\in(0, 1)$ is a hyper parameter and $N$ is the number of elapsed episodes.

The following corollary is an immediate consequence of Theorems \ref{theory:upper} and \ref{theory:lower} together with Assumption \ref{as:1}.
\begin{corollary}\label{cor:PolyRL-LSAW}
Given that Assumption \ref{as:1} is satisfied, any exploratory trajectory induced by PolyRL Algorithm (ref. Algorithm \ref{alg:plyrl}) with high probability is an LSA-RW.
\end{corollary}
The proof is provided and discussed in the Appendix (Section \ref{sec:proofs}).

\vspace{-0.1in}
\section{Related Work}\label{sec:related work}

A wide range of exploration techniques with theoretical guarantees (\emph{e.g.} PAC bounds) have been developed for MDPs with finite or infinitely countable state or action spaces \cite{kearns2002near, brafman2002r, strehl2004exploration,lopes2012exploration, azar2017minimax, white2010interval, Wang2020Q-learning}, however extending these algorithms to real-world settings with continuous state and action spaces without any assumption on the structure of state-action spaces or the reward function is impractical \cite{haarnoja2017reinforcement,VIME}.

Perturbation-based exploration strategies are, by nature, agnostic to the structure of the underlying state-action spaces and are thus suitable for continuous domains. Classic perturbation-based exploration strategies typically involve a perturbation mechanism at either the action-space or the policy parameter-space level. These methods subsequently employ stochasticity at the policy level as the main driving force for exploration \cite{deisenroth2013survey}. Methods that apply perturbation at the parameter level often preserve the noise distribution throughout the trajectory  \cite{ng2000pegasus,sehnke2010parameter,theodorou2010generalized, fortunato2017noisy, colas2018gep,plappert2018parameter}, and do not utilize the trajectory information in this regard. Furthermore, a majority of action-space perturbation approaches employ per-step independent or correlated perturbation \cite{wawrzynski2009real, lillicrap2015continuous,haarnoja2017reinforcement,xu2018learning}. For instance, \cite{lillicrap2015continuous} uses the Ornstein–Uhlenbeck (OU) process to produce auto-correlated additive noise at the action level and thus benefit from the correlated noise between consecutive actions. 

While maintaining correlation between consecutive actions is advantageous in many locomotion tasks \cite{morimoto2001acquisition, kober2013reinforcement,gupta2018meta}, it brings technical challenges due to the non-Markovian nature of the induced decision process \cite{perez2017non}, which leads to substantial dependence on the history of the visited states. This challenge is partially resolved in the methods that benefit from some form of parametric memory \cite{plappert2018parameter, lillicrap2015continuous, sharma2017learning}. However, they all suffer from the lack of \textit{informed orientational persistence}, \emph{i.e.} the agent's persistence in selecting actions that preserve the orientation of the state trajectory induced by the target policy. Particularly in sparse-reward structured environments, where the agent rarely receives informative signals, the agent will eventually get stuck in a small region since the analytical change on the gradient of greedy policy is minimal \cite{hare2019dealing, kang2018policy}. Hence, the agent's persistent behaviour in action sampling with respect to the local history (short-term memory) of the selected state-action pairs plays a prominent role in exploring the environments with sparse or delayed reward structures.

\begin{figure*}[h!]
\begin{center}
\includegraphics[width=\textwidth]{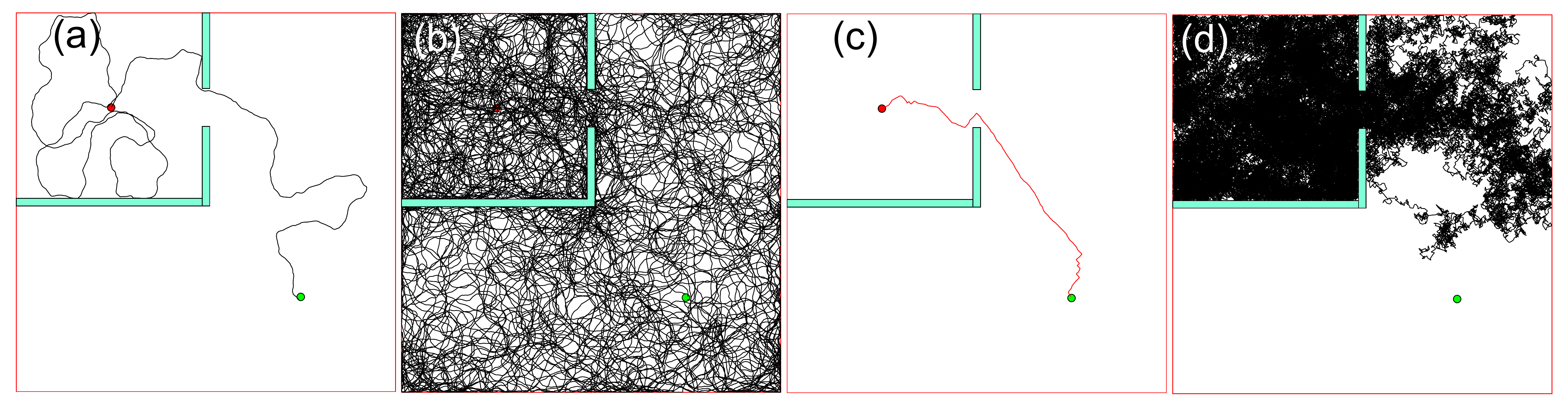}
\end{center}
\caption{Visualization of the agent's trajectory in a 2D navigation task with continuous state-action space and sparse-reward structure. The environment is composed of a $100\times100$ big chamber and a $50\times50$ smaller chamber. The start state is located inside the small chamber (the red circle with a diameter of $1$), and the goal is outside the small chamber (the green circle with a diameter of $1$). The agent receives $+100$ reward when it reaches the goal and $0$ reward elsewhere. (a-c) Performance of the PolyRL agent with the correlation angle $\theta \simeq 0.2$. (a) A sample trajectory of the PolyRL agent in one episode. (b) Coverage of the environment after $11$ episodes. (c) A sample trajectory of the PolyRL agent during the evaluation phase after learning the task. (d) The $\epsilon$-greedy agent's coverage of the environment after $11$ episodes.}
\label{fig:GridWorld}
\end{figure*}

\section{Experiments}\label{sec:experiment}

We assess the performance of PolyRL in comparison with that of other state-of-the-art exploration approaches combined with off-policy learning methods. In particular, we integrate PolyRL with three learning algorithms: (1) the Q-learning method (\cite{watkins1992q}) with linear function approximation in a 2D sparse continuous state-and-action navigation task, where the performance of PolyRL is compared with that of $\epsilon$-greedy; (2) the Soft Actor-Critic (SAC) algorithm \cite{haarnoja2018soft}, where PolyRL is combined with SAC (SAC-PolyRL) and replaces the random exploration phase during the first $10,000$ steps in the SAC algorithm and is subsequently compared with SAC as well as Optimistic Actor-Critic (OAC) \cite{ciosek2019better} methods; and (3) the deep deterministic policy gradients (DDPG) \cite{lillicrap2015continuous} algorithm, where PolyRL (DDPG-PolyRL) is assessed in comparison with additive uncorrelated Gaussian action space noise (DDPG-UC), correlated Ornstein-Uhlenbeck action space noise (DDPG-OU) \cite{uhlenbeck1930theory, lillicrap2015continuous}, adaptive parameter space noise (DDPG-PARAM) \cite{plappert2017parameter}, as well as the Fine Grained Action Repetition (DDPG-FiGAR) \cite{sharma2017learning} method.

The sets of experiments involving the learning methods DDPG and SAC are performed in MuJoCo high-dimensional continuous control tasks ``SparseHopper-V2'' ($\actions\subset\mathbb{R}^3$, $\states\subset\mathbb{R}^{11}$), ``SparseHalfCheetah-V2'' ($\actions\subset\mathbb{R}^6$, $\states\subset\mathbb{R}^{17}$), and ``SparseAnt-V2'' ($\actions\subset\mathbb{R}^8$, $\states\subset\mathbb{R}^{111}$)  (Refer to the Appendix (Section \ref{sec:aditional-baseline}) for the benchmark results of the same algorithms in the standard (dense-reward) MuJoCo tasks).

\textbf{Algorithm and Environment Settings -} The environment in our 2D sparse-reward navigation tasks either consists of only one $400\times400$ chamber (goal reward $+1000$), or a $50\times50$ room encapsulated by a $100\times100$ chamber. Initially positioned inside the small room, the agent's goal in the latter case is to find its way towards the bigger chamber, where the goal is located (goal reward $+100$) (Figure \ref{fig:GridWorld}). Moreover, in order to make the former task more challenging, in a few experiments, we introduce a \emph{puddle} in the environment, where upon visiting, the agent receives the reward $-100$. In order to assess the agent's performance, we integrate the PolyRL exploration algorithm with the Q-learning method with linear function approximation (learning rate $= 0.01$) and compare the obtained results with those of the $\epsilon$-greedy exploration with Q-learning. We subsequently plot the quantitative results for the former task and visualize the resulting trajectories for the latter. 

In the sparse MuJoCo tasks, the agent receives a reward of $+1$ only when it crosses a target distance $\lambda$, termed the \emph{sparsity threshold}. Different $\lambda$ values can change the level of difficulty of the tasks significantly. Note that due to the higher performance of SAC-based methods compared with that of DDPG-based ones, we have elevated $\lambda$ for SAC-based experiments, making the tasks more challenging. Moreover, we perform an exhaustive grid search over the corresponding hyper parameters for each task. The sparsity thresholds, the obtained hyper-parameter values, as well as the network architecture of the learning algorithms DDPG and SAC are provided in the Appendix (Section \ref{sec:hyperparams}).
\begin{figure*}[h!]
\begin{center}
\includegraphics[width=\textwidth]{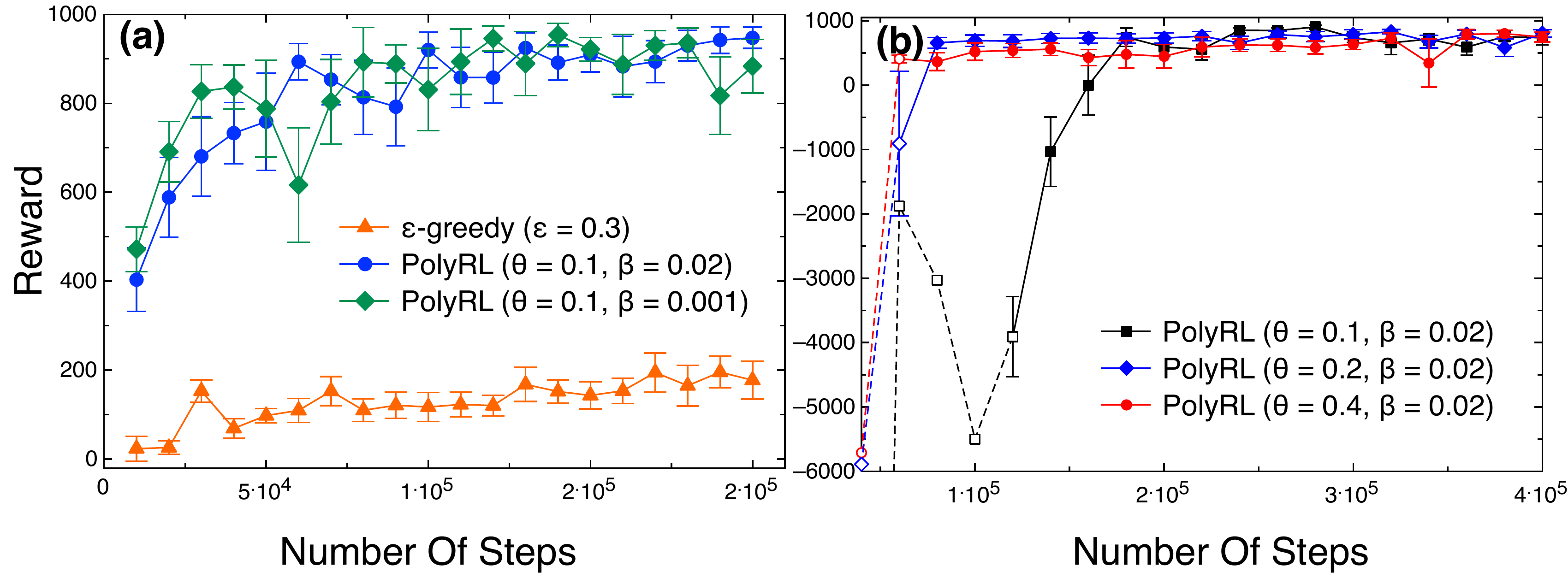}
\end{center}
\caption{The performance of PolyRL and $\epsilon$-greedy in a 2D continuous sparse-reward navigation task (the $400\times400$ chamber). (a) The environment does not include a puddle. (b) The environment contains a puddle. The first point on the $\theta = 0.1$ curve is not shown on the graph due to its minimal average reward value ($\simeq-30000$), which affects the range of the y-axis. It is connected to the next point via a dashed line. The data points on the dashed lines do not include error bars because of substantial variances. For the same reason (out of range values and large variances), the results for the $\epsilon$-greedy exploration method are not shown here.}
\label{fig:2DPlot}
\end{figure*}

\begin{figure*}[h!]
\begin{center}
\includegraphics[width=\textwidth]{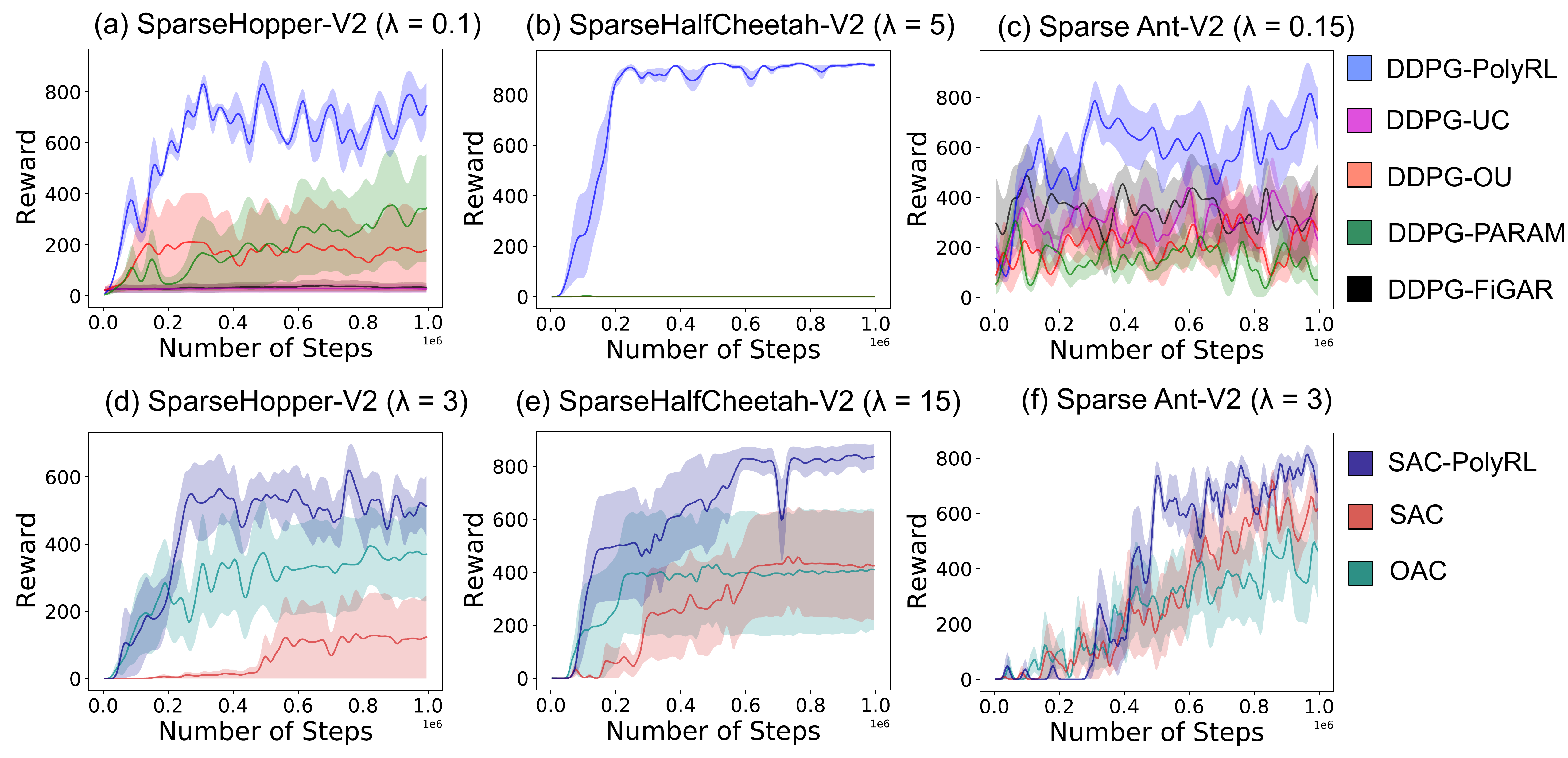}
\end{center}
\caption{The simulation results on ``SparseHopper-v2'' (a and d), ``SparseHalfCheetah-v2'' (b and e) and ``SparseAnt-v2'' (c and f). The results are averaged over $5$ random seeds. The error bars depict the standard error on the mean (Refer to the Appendix (Section \ref{sec:aditional-baseline}) for the benchmark results obtained from the baseline algorithms).}
\label{fig:result-run}
\end{figure*}

\textbf{Results and Discussion -} We present the qualitative results for the 2D sparse navigation task in Figure \ref{fig:GridWorld}. An example of PolyRL trajectory in one episode (Figure \ref{fig:GridWorld} (a)) demonstrates the expansion of the PolyRL agent trajectory in the environment. After $11$ episodes, the PolyRL agent exhibits a full coverage of the environment (Figure \ref{fig:GridWorld} (b)) and is consequently able to learn the task (Figure \ref{fig:GridWorld} (c)), while the $\epsilon$-greedy agent is not able to reach the goal state even once, and thus fails to learn the task (Figure \ref{fig:GridWorld} (d)). This visual observation highlights the importance of space coverage in sparse-reward tasks, where the agent rarely receives informative reinforcement from the environment. An effective trajectory expansion in the environment exposes the agent to the unvisited regions of the space, which consequently increases the frequency of receiving informative reinforcement and accelerates the learning process. In Figure \ref{fig:2DPlot}, the quantitative results for learning the task in a similar environment (the $400\times400$ chamber) are shown in both the absence (Figure \ref{fig:2DPlot} (a)) and presence (Figure \ref{fig:2DPlot} (b)) of a puddle. In both cases, the PolyRL exploration method outperforms $\epsilon$-greedy. In Figure \ref{fig:2DPlot} (b), we observe that trajectories with lower persistence (larger $\theta$) present a better performance compared with stiffer trajectories. Note that due to the existence of walls in these 2D tasks, the transition probability kernel is specifically non-smooth at the walls. Thus, the smoothness assumption on the transition probability kernel made earlier for the theoretical convenience does not apply in these specific environments. Yet, we empirically show that PolyRL still achieves a high performance in learning these tasks.

\begin{figure*}[h!]
\begin{center}
\includegraphics[width=\textwidth]{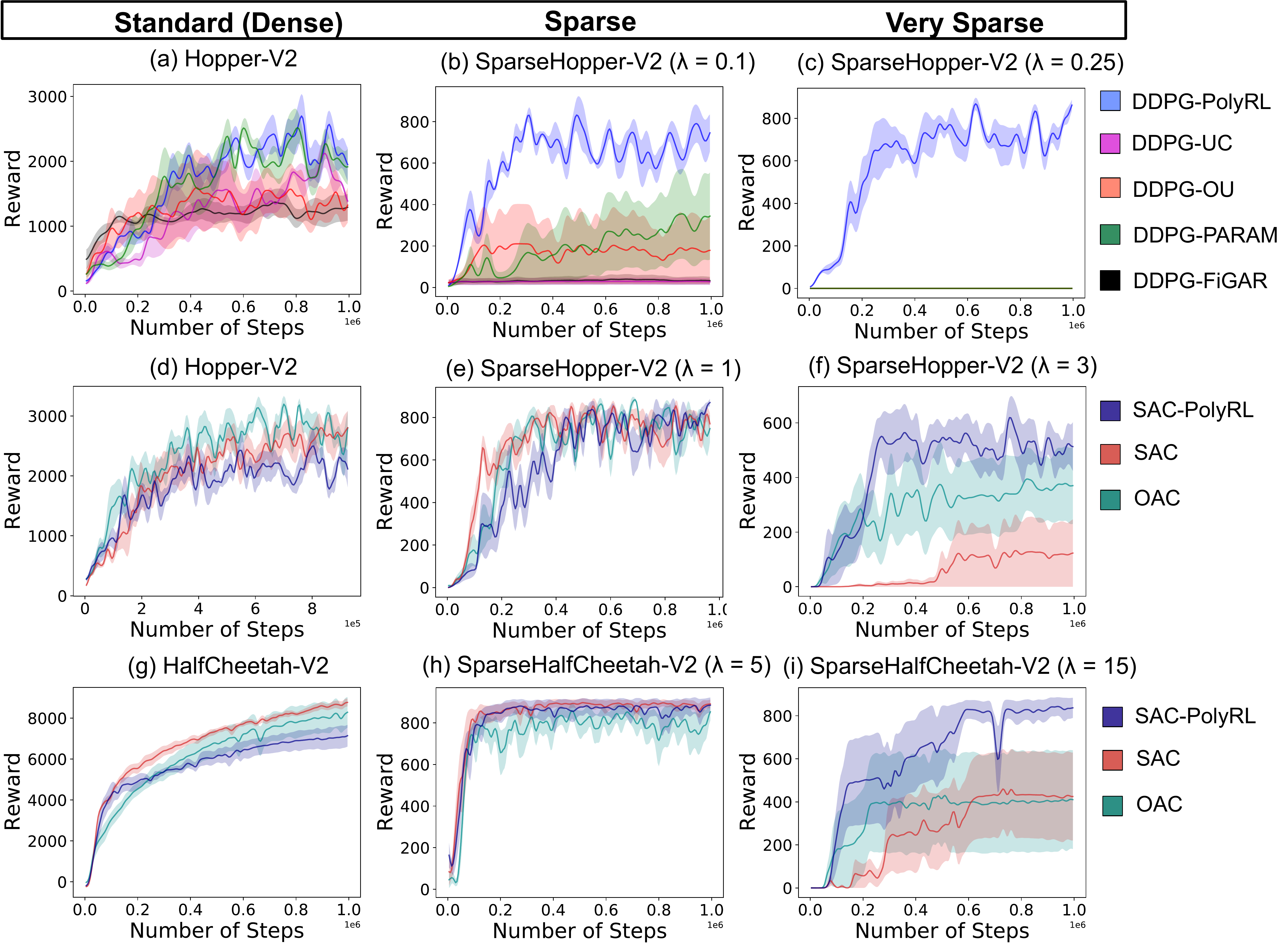}
\end{center}
\caption{PolyRL sensitivity to the change of sparsity threshold in some sample experiments. The column on the left depicts the results for standard (with dense reward) MuJoCo tasks. From left to right, the sparsity of the reward structure increases.}
\label{fig:sensitivity}
\end{figure*}
We illustrate the plotted results for the sparse MuJoCo tasks in Figure \ref{fig:result-run}. The obtained results for DDPG-based and SAC-based algorithms in SparseHopper-V2 (Figures \ref{fig:result-run} (a,~d)), SparseHalfCheetah-V2 (Figures \ref{fig:result-run} (b,~e)), and SparseAnt-V2 (Figures \ref{fig:result-run} (c,~f)) show that integrating the two learning methods with the exploration algorithm PolyRL leads to improvement in their performance. The results achieved by PolyRL exploration method confirms that the agent has been able to efficiently and sufficiently cover the space, receive informative reinforcement, and learn the tasks. The high performance of SAC-PolyRL is particularly significant, in the sense that PolyRL assists SAC in the data generation process only for the first $10,000$ steps. Yet, this short presence leads to a notable contribution in enhancing the performance of SAC.

Another notable feature of PolyRL is its relatively low sensitivity to the increase in sparsity threshold $\lambda$ compared to that of DDPG-based and SAC-based algorithms. Figure \ref{fig:sensitivity} illustrates the performance of PolyRL in some sample experiments for three different sparsity thresholds. The level of complexity of the tasks increases from left to right with the sparsity threshold $\lambda$. As the sparsity level changes from sparse to very sparse, the graphs demonstrate a sharp decrease in the performance of PolyRL counterparts, while the PolyRL performance stays relatively stable (Note that due to the reward structure in these sparse tasks and considering that the maximum number of steps in each episode is by default set to  $1000$ in the Gym environments, the maximum reward that an agent could get during each evaluation round cannot exceed $1000$). This PolyRL agent's behaviour can be explained by its relatively fast expansion in the space (Refer to Proposition 2 in the Appendix (Section \ref{sec:poly-models})), which leads to faster access to the sparsely distributed rewards compared with other DDPG-based and SAC-based methods. On the other hand, in standard tasks, where the reinforcement is accessible to the agents at each time-step, as PolyRL does not use the received rewards in its action selection process, it might skip the informative signals nearby and move on to the farther regions in the space, leading to possibly acquiring less amount of information and lower performance. In other words, the strength of PolyRL is most observable in the tasks where accessing information is limited or delayed.

\section{Conclusion}\label{sec:conclusion}

We propose a new exploration method in reinforcement learning (PolyRL), which leverages the notion of locally self-avoiding random walks and is designed for environments with continuous state-action spaces and sparse-reward structures. The most interesting aspect of our proposal is the explicit construction of each exploratory move based on the entire existing trajectory, rather than just the current observed state. While the agent chooses its next move based on its current state, the inherent locally self-avoiding property of the walk acts as an implicit memory, which governs the agent's exploratory behaviour. Yet this locally controlled behavior leads to an interesting global property for the trajectory, which is an improvement in the coverage of the environment. This feature, as well as not relying on extrinsic rewards in the decision-making process, makes PolyRL perfect for the sparse-reward tasks. We assess the performance of PolyRL in 2D continuous sparse navigation tasks, as well as three sparse high-dimensional simulation tasks, and show that PolyRL performs significantly better than the other exploration methods in combination with the baseline learning algorithms DDPG and SAC. Moreover, we show that compared to the performance of PolyRL counterparts, the performance of our proposed method has lower sensitivity to the increase in the sparsity of the extrinsic reward, and thus its performance stays relatively stable. Finally, a more adaptive version of PolyRL, which can map the changes in the action trajectory stiffness to that of the state trajectory, could be helpful in more efficient learning of the simulation tasks.

\section*{Acknowledgements}
The authors would like to thank Prof. Walter Reisner for providing valuable feedback on the initial direction of this work, and Riashat Islam for helping with the experiments in the early stages of the project. Computing resources were provided by Compute Canada and Calcul Qu\'ebec throughout the project, and by Deeplite Company for preliminary data acquisition, which the authors appreciate. Funding is provided by Natural Sciences and Engineering Research Council of Canada (NSERC).

\bibliography{References}
\bibliographystyle{icml2021}
\clearpage

\onecolumn

\appendix\label{sec:App}
\begin{center}\large\textbf{Appendix}\end{center}
Here, we provide additional information on different parts of the paper. In particular, in section \ref{sec:poly-models} we introduce and discuss two chain models in polymer physics. In section \ref{sec:proofs}, we provide the theoretical proofs of Theorems \ref{theory:upper} and \ref{theory:lower}, Lemma \ref{lem:bond-transform}, and Corollary \ref{cor:PolyRL-LSAW} in the manuscript. In section \ref{sec:action-sampling}, we present the action-sampling algorithm, and in section \ref{sec:aditional-baseline} we provide additional baseline results in the standard MuJoCo tasks. Finally, in section \ref{sec:hyperparams}, we provide the network architecture of the learning methods, as well as the PolyRL hyper parameters used in the experimental section.

\section{Polymer Models}\label{sec:poly-models}
In the field of \emph{Polymer Physics}, the conformations and interactions of polymers that are subject to thermal fluctuations are modeled using principles from statistical physics. In its simplest form, a polymer is modeled as an \emph{ideal chain}, where interactions between chain segments are ignored. The \emph{no-interaction} assumption allows the chain segments to cross each other in space and thus these chains are often called \emph{phantom chains} \cite{doi1988theory}. In this section, we give a brief introduction to two types of ideal chains.

\begin{figure*}[ht]
\begin{center}
\includegraphics[width=1 \textwidth]{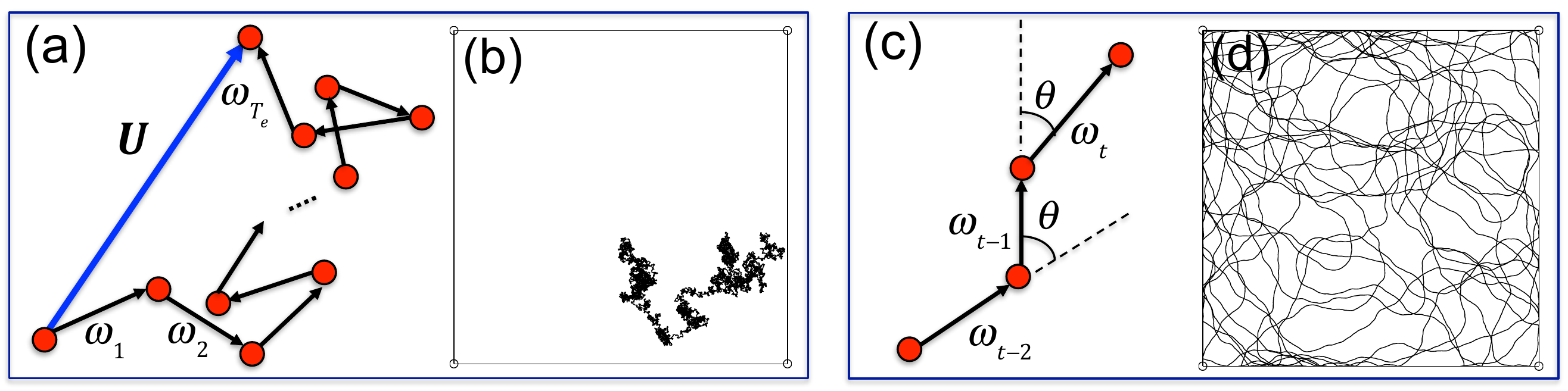}
\end{center}
\caption{ A chain (or trajectory) is shown as a sequence of $T_e$ random bond vectors $\{\boldsymbol{\omega}_i\}_{i=1..T_e}$. In a freely-jointed chain (a), the orientation of the bond vectors are independent of one another. The end-to-end vector of the chain is depicted by $\boldsymbol{U}$. In a freely-rotating chain (c), the correlation angle $\theta$ is invariant between every two consecutive bond vectors, which induces a finite stiffness in the chain. (b, d) A qualitative comparison between an FJC (b) and an FRC with $\theta\approx5.7^\circ$ (d), in a 2D environment of size $400\times400$ for $20000$ number of moves.}
\label{fig:Polymer-def-supp}
\end{figure*}

Two main ideal chain models are: 1) freely-jointed chains (FJCs) and 2) freely-rotating chains (FRCs) \cite{doi1988theory}. In these models, chains of size $T_e$ are demonstrated as a sequence of $T_e$ random vectors $\{ \boldsymbol{\omega}_i\}_{i=1..T_e}$, which are as well called \emph{bond vectors} (See Figure~\ref{fig:Polymer-def-supp}). FJC is the simplest proposed model and is composed of mutually independent random vectors of the same size (Figure~\ref{fig:Polymer-def-supp}(a)). In other words, an FJC chain is formed via uniform random sampling of vectors in space, and thus is a \emph{random walk} (RW). In the FRC model, on the other hand, the notion  of \emph{correlation angle} is introduced, which is the angle $\theta$ between every two consecutive bond vectors. The FRC model, fixes the correlation angle $\theta$ (Figure~\ref{fig:Polymer-def-supp}(c)), thus the vectors in the chain are temporally correlated. The vector sampling strategy in the FRCs induces \emph{persistent chains}, in the sense that the orientation of the consecutive vectors in the space are preserved for certain number of time steps (a.k.a. persistence number), after which the correlation is broken and the bond vectors \emph{forget} their original orientation. This feature introduces a finite \emph{stiffness} in the chain, which induces what we call {\em local} self avoidance, leading to faster expansion of the chain in the space (Compare Figures \ref{fig:Polymer-def-supp}~(b) and (d) together). Below, we discuss two important properties of the FJCs and the FRCs, and subsequently formally introduce the \emph{locally self-avoiding random walks} (LSA-RWs) in Definition 1.\\

\textbf{FJCs} (\emph{Property}) \textbf{-} In the Freely-Jointed Chains (FJCs) or the flexible chains model, the orientations of the bond vectors in the space are mutually independent. To measure the expected end-to-end length of a chain $\tilde{U}$ with $T_e$ bond vectors of constant length $b_o$ given the end-to-end vector ${\bf U} = \sum_{i=1}^{T_e} \boldsymbol{\omega}_i$ (Figure \ref{fig:Polymer-def-supp} (a)) and considering the mutual independence between bond vectors of an FJC, we can write \cite{doi1988theory},

\begin{align}
        \mathbb{E}[\norm{{\bf U}}^2] =\sum_{i,j=1}^{T_e}\mathbb{E}[\boldsymbol{\omega}_i.\boldsymbol{\omega}_j] = \sum_{i=1}^{T_e}\mathbb{E}[\boldsymbol{\omega}_i^2] + 2\sum_{i> j}\mathbb{E}[\boldsymbol{\omega}_i.\boldsymbol{\omega}_j] = T_e b_o^2,\label{eq:FJC}
\end{align}
where $\mathbb{E}[.]$ denotes the ensemble average over all possible conformations of the chain as a result of thermal fluctuations. Equation \ref{eq:FJC} shows that the expected end-to-end length $\tilde{U} = \mathbb{E}[\norm{{\bf U}}^2]^{1/2} = b_o\sqrt{T_e}$, which reveals random-walk behaviour as expected. 

\textbf{FRCs} (\emph{Property}) \textbf{-}
In the Freely-Rotating Chains (FRCs) model, we assume that the angle $\theta$ (correlation angle) between every two consecutive bond vectors is invariant (Figure \ref{fig:Polymer-def-supp} (c)). Therefore, bond vectors ${\boldsymbol{\omega}}_{i:1,...,T_e}$ are not mutually independent. Unlike the FJC model, in the FRC model the bond vectors are correlated such that \cite{doi1988theory},
\begin{align}
    \mathbb{E}[\boldsymbol{\omega}_i.\boldsymbol{\omega}_j]=b_o^2\left(\cos\theta\right)^{|i-j|}=b_o^2e^{-\frac{|i-j|}{Lp}},\label{eq:FRC-correlation}
\end{align}
where $Lp=\frac{1}{|\log(\cos\theta)|}$ is the correlation length (persistence number). Equation \ref{eq:FRC-correlation} shows that the correlation between bond vectors in an FRC is a decaying exponential with correlation length $Lp$. 

\begin{customlemma}{1}\label{thm:effective_length_lemma} \cite{doi1988theory}
Given an FRC characterized by end-to-end vector $\boldsymbol{U}$, bond-size $b_o$ and number of bond vectors $T_e$, we have $\mathbb{E}[\norm{\boldsymbol{U}}^2]=b^2T_e$, where $b^2 = b^2_o \frac{1+cos \theta}{1-cos \theta}$ and $b$ is called the \textit{effective bond length}.
\end{customlemma}
Lemma \ref{thm:effective_length_lemma} shows that FRCs obey random walk statistics with step-size (bond length) $b>b_o$. The ratio $b/b_o = \frac{1+\cos\theta}{1-\cos\theta}$ is a measure of the stiffness of the chain in an FRC. 

FRCs have high expansion rates compared to those of FJCs, as presented in Proposition \ref{prop:expanding} below.
\begin{customproposition}{2}[Expanding property of LSA-RW] \cite{doi1988theory}\label{prop:expanding}
Let $\tau$ be a LSA-RW with the persistence number $Lp_{\tau} > 1$ and the end-to-end vector $\boldsymbol{U}(\tau)$, and let $\tau^\prime$ be a random walk (RW) and the end-to-end vector $\boldsymbol{U}(\tau^\prime)$. Then for the same number of time steps and same average bond length for $\tau$ and $\tau^\prime$, the following relation holds,
\begin{align}
    \frac{\mathbb{E}[\norm{\boldsymbol{U}(\tau)}]}{\mathbb{E}[\norm{\boldsymbol{U}(\tau^\prime)}]} = \frac{1+e^{-1/Lp_\tau}}{1-e^{-1/Lp_\tau}} >1,\label{eq:expanding}
\end{align}
where the persistence number $Lp_\tau = \frac{1}{|\log\cos\theta|}$, with $\theta$ being the average correlation angle between every two consecutive bond vectors.
\end{customproposition}
\begin{proof}
This proposition is the direct result of combining Equations $2.7$ and $2.14$ in \cite{doi1988theory}. Equation $2.7$ provides the expected $T_e$ time-step length of the end-to-end vector with average step-size $b_o$ associated with FJCs and Equation $2.14$ provides a similar result for FRCs. Note that in the FRC model, since the bond vectors far separated in time on the chain are not correlated, they can cross each other.

\end{proof}
\textbf{Radius of Gyration} (\textit{Formal Definition}) \cite{rubenstein2003polymer} 
The square radius of gyration $U_g^2(\tau)$ of a chain $\tau$ of size $T_e$ is defined as the mean square distance between position vectors $\boldsymbol{t}\in\tau$ and the chain center of mass ($\bar{\boldsymbol{\tau}}$), and is written as,
\begin{align}
    U_g^2(\tau) := \frac{1}{T_e}\sum_{i=1}^{T_e} ||\boldsymbol{t}_i-\bar{\boldsymbol{\tau}}||^2,\label{eq:gyration}
\end{align}
where $\bar{\boldsymbol{\tau}} = \frac{1}{T_e} \sum_{i=1}^{T_e} \boldsymbol{t}_i$. When it comes to selecting a measure of coverage in the space where the chain resides, radius of gyration $U_g$ is a more proper choice compared with the end-to-end distance $||\boldsymbol{U}||$, as it signifies the size of the chain with respect to its center of mass, and is proportional to the radius of the sphere (or the hyper sphere) that the chain occupies. Moreover, in the case of chains that are circular or branched, and thus cannot be assigned an end-to-end length, radius of gyration proves to be a suitable measure for the size of the corresponding chains \cite{rubenstein2003polymer}. For the case of fluctuating chains, the square radius of gyration is usually ensemble averaged over all possible chain conformations, and is written as \cite{rubenstein2003polymer},
\begin{align}
    \mathbb{E}[U_g^2(\tau)] := \frac{1}{T_e}\sum_{i=1}^{T_e} \mathbb{E}[||\boldsymbol{t}_i-\bar{\boldsymbol{\tau}}||^2].\label{eq:gyration2}
\end{align}

\begin{remark} \label{rem:1}
The square radius of gyration $U_g^2$ is proportional to the square end-to-end distance $||\boldsymbol{U}||^2$ in ideal chains (\emph{e.g.} FJCs and FRCs) with a constant factor \cite{rubenstein2003polymer}. Thus, Proposition \ref{prop:expanding} and Equation \ref{eq:expanding}, which compare the the end-to-end distance of LSA-RW and RW with each other, similarly hold for the radius of gyration of the respective models, implying faster expansion of the volume occupied by LSA-RW compared with that of RW.
\end{remark}

\section{The Proofs}\label{sec:proofs}
In this section, the proofs for the theorems and Lemma \ref{lem:bond-transform} in the manuscript are provided.

\subsection{The proof of Lemma \ref{lem:bond-transform} in the manuscript}
\textbf{Lemma 2 statement}: Let $\tau_\states = (s_0, \dots, s_{T_e-1})$ be the trajectory of visited states, $s_{T_e}$ be a newly visited state and $\boldsymbol{\omega}_i = s_i-s_{i-1}$ be the bond vector that connects two consecutive visited states $s_{i-1}$ and $s_{i}$. Then we have,
 \begin{align}
      \norm{s_{T_e}- \bar{\tau}_\states}^2 = \norm{\boldsymbol{\omega}_{T_e} + \frac{1}{T_e} \left[ \sum_{i=1}^{T_e-1} i \boldsymbol{\omega}_i \right]}^2.\label{eq:lem-bond-transform1}
 \end{align}
\begin{proof}\label{proof:lem-bond-transform}

Using the relation $\Bar{\tau}_{\states}:=\frac{1}{T_e}\sum_{s \in \tau_\states} s$ as well as the definition of bond vectors (Equation \ref{eq:bond-vector} in the manuscript), we can write $s_{T_e}-\bar{\tau}_\states$ on the left-hand side of Equation \ref{eq:lem-bond-transform1} as,

\begin{align}
     s_{T_e}- \bar{\tau}_\states =& s_{T_e} - \frac{1}{T_e} \sum_{s \in \tau_\states} s\nonumber\\
     =& s_{T_e}-s_{T_e-1}+s_{T_e-1}-\frac{1}{T_e} \sum_{s \in \tau_\states} s\nonumber \\ 
    =&\boldsymbol{\omega}_{T_e} + \frac{1}{T_e} (s_{T_e-1}-s_0) + (s_{T_e-1}-s_1) + (s_{T_e-1}-s_2)+ \dots\nonumber\\
   &+(s_{T_e-1}-s_{T_e-2})]\nonumber\\ 
   =& \boldsymbol{\omega}_{T_e} + \frac{1}{T_e} [(s_{T_e-1}-s_{T_e-2}+s_{T_e-2}-s_{T_e-3}+\dots\nonumber\\
    &+s_2-s_1 + s_1-s_0)+(s_{T_e-1}-s_{T_e-2}+s_{T_e-2}-s_{T_e-3}+\dots\nonumber\\ &+s_3-s_2+s_2-s_1)+ \dots +(s_{T_e-1}-s_{T_e-2})]\nonumber\\
     =&\boldsymbol{\omega}_{T_e} + \frac{1}{T_e} \left[ (\boldsymbol{\omega}_{T_e-1}+\dots+\boldsymbol{\omega}_1)+(\boldsymbol{\omega}_{T_e-1} + \dots +\boldsymbol{\omega}_2)+\dots + \boldsymbol{\omega}_{T_e-1} \right]\nonumber\\
         =&\boldsymbol{\omega}_{T_e} + \frac{1}{T_e} \left[ \sum_{i=1}^{T_e-1} i \boldsymbol{\omega}_i \right]
\end{align}
\begin{align}
    &\Rightarrow \norm{s_{T_e}-\bar{\tau}_\states}^2 =\norm{\boldsymbol{\omega}_{T_e} + \frac{1}{T_e} \left[ \sum_{i=1}^{T_e-1} i \boldsymbol{\omega}_i \right]}^2\label{eq:proof-lem-bond-transform}
\end{align}
\end{proof}
\subsection{The proof of Theorem \ref{theory:upper} in the manuscript}

\textbf{Theorem 3 statement} (Upper-Bound Theorem)%\label{theory:upper}
Let $\beta \in (0,1)$ and $\tau_\states$ be an LSA-RW in $\states$ induced by PolyRL with the persistence number $Lp_{\tau_\states}>1$ within episode $N$, $\omega_{\tau_\states}=\{\boldsymbol{\omega}_i\}_{i=1}^{T_e-1}$ be the sequence of corresponding bond vectors, where $T_e>0$ denotes the number of bond vectors within $\tau_\states$, and $b_o$ be the average bond length. The upper confidence bound for $ULS_{Ug^2}(\tau_\states)$ with probability of at least $1-\delta$ is, 

\begin{align}%\label{eq:upper}
  UB = &\Lambda (T_e,\tau_\states) +\frac{1}{\delta}{\left[ \Gamma (T_e, b_o, \tau_\states) + \frac{2b^2_o}{T_e^2} \sum_{i=1}^{T_e-1} i  e^{\frac{-(T_e-i)}{Lp_{\tau_\states}}} \right]},
\end{align}
where,
\begin{align}
    \Lambda (T_e,\tau_\states) = -\frac{1}{T_e-1} U_g^2(\tau_\states) \\
    \Gamma (T_e, b_o, \tau_\states) = \frac{b^2_o}{T_e}  + \frac{\norm{\sum_{i=1}^{T_e-1} i\boldsymbol{\omega}_i}^2}{T_e^3}
\end{align}

\begin{proof}
If we replace the term ${U_g}^2(\tau^{\prime}_\states)$ in Equation \ref{eq:upper_bound_def} in the manuscript with its incremental representation as a function of ${U_g}^2(\tau_\states)$, we get
\begin{align}
    ULS_{Ug^2}(\tau_\states)&= \sup_{s_{T_e} \in \Omega}\left( \frac{T_e-2}{T_e-1} {U_g}^2(\tau_\states) + \frac{1}{T_e} \norm{s_{T_e}-\Bar{\tau}_\states}^2-{U_g}^2(\tau_\states)\right)\label{eq:incremental_variance}\nonumber\\
    &=  -\frac{1}{T_e-1} {U_g}^2(\tau_\states) + \sup_{s_{T_e} \in \Omega}  \frac{1}{T_e} \norm{s_{T_e}-\Bar{\tau}_\states}^2.
\end{align}
Therefore, the problem reduces to the calculation of 
\begin{align}\label{eq:upper-bound}
   \frac{1}{T_e} \sup_{s_{T_e} \in \Omega}   \norm{s_{T_e}-\Bar{\tau}_\states}^2.
\end{align}
Using Lemma \ref{lem:bond-transform} in the manuscript, we can write Equation \eqref{eq:upper-bound} in terms of bond vectors $\boldsymbol{\omega}_i = s_i-s_{i-1}$ as,
\begin{align}
    \frac{1}{T_e} \sup_{s_{T_e} \in \Omega}   \norm{s_{T_e}-\Bar{\tau}_\states}^2 = \frac{1}{T_e} \sup_{s_{T_e} \in \Omega} \norm{\boldsymbol{\omega}_{T_e} + \frac{1}{T_e} \left[ \sum_{i=1}^{T_e-1} i \boldsymbol{\omega}_i\right]}^2.
\end{align}
From now on, with a slight abuse of notation, we will treat $\boldsymbol{\omega}_{T_e}=S_{T_e}-s_{T_e-1}$ as a random variable due to the fact that $S_{T_e}$ is a random variable in our system. Note that $\boldsymbol{\omega}_i$ for $i=1,2,\dots,T_e-1$ is fixed, and thus is not considered a random variable. We use high-probability concentration bound techniques to calculate Equation \eqref{eq:upper-bound}. For any $\delta \in (0,1)$, there exists $\alpha>0$, such that
\begin{align}
    &\Pr[\norm{\boldsymbol{\omega}_{T_e} + \frac{1}{T_e} \left[ \sum_{i=1}^{T_e-1} i \boldsymbol{\omega}_i \right]}^2 < \alpha | S_{T_e} \in \Omega] > 1-\delta.\label{eq:upper_conc_bound}
\end{align}
We can rearrange Equation \ref{eq:upper_conc_bound} as,
\begin{align}
    &\Pr[\norm{\boldsymbol{\omega}_{T_e} + \frac{1}{T_e} \left[ \sum_{i=1}^{T_e-1} i \boldsymbol{\omega}_i \right]}^2 \geq \alpha | S_{T_e} \in \Omega] \leq \delta.\label{eq:upper_conc_bound1}
\end{align}
Multiplying both sides by $T_e^2$ and expanding the squared term in Equation \ref{eq:upper_conc_bound1} gives,
\begin{align}
    &\Pr[T_e^2 \norm{\boldsymbol{\omega}_{T_e}}^2 + 2 T_e (\boldsymbol{\omega}_{T_e}.\sum_{i=1}^{T_e-1} i\boldsymbol{\omega}_i)+\norm{\sum_{i=1}^{T_e-1} i\boldsymbol{\omega}_i}^2 \geq T_e^2 \alpha| S_{T_e} \in \Omega] \leq \delta.
\end{align}
By Markov’s inequality we have, 
\begin{align}
    \Pr&\left[T_e^2 \norm{\boldsymbol{\omega}_{T_e}}^2 + 2 T_e( \boldsymbol{\omega}_{T_e}.\sum_{i=1}^{T_e-1} i\boldsymbol{\omega}_i) + \norm{\sum_{i=1}^{T_e-1} i\boldsymbol{\omega}_i}^2 \geq T_e^2 \alpha \right]\nonumber\\
    &\leq \frac{\ex{T_e^2 \norm{\boldsymbol{\omega}_{T_e}}^2 + 2 T_e (\boldsymbol{\omega}_{T_e}.\sum_{i=1}^{T_e-1} i\boldsymbol{\omega}_i)+\norm{\sum_{i=1}^{T_e-1} i\boldsymbol{\omega}_i}^2}}{T_e^2 \alpha }= \delta\nonumber\\
    \implies & \alpha = \frac{1}{\delta T_e^2}{\left[T_e^2\ex{ \norm{\boldsymbol{\omega}_{T_e}}^2} + 2T_e \ex{\boldsymbol{\omega}_{T_e} . \sum_{i=1}^{T_e-1} i\boldsymbol{\omega}_i} +\norm{\sum_{i=1}^{T_e-1} i\boldsymbol{\omega}_i}^2   \right]} \nonumber\\
    \underbrace{\implies}_{\text{by Def. 1}} & \alpha = \frac{1}{\delta T_e^2}{\left[T_e^2 b^2_o + 2T_e \ex{\boldsymbol{\omega}_{T_e} . \sum_{i=1}^{T_e-1} i\boldsymbol{\omega}_i} +\norm{\sum_{i=1}^{T_e-1} i\boldsymbol{\omega}_i}^2   \right]}\nonumber
\end{align}

Note that all expectations $\mathbb{E}$ in the equations above are over the transition kernel $\mathcal{P}$ of the MDP. Using the results from Lemma \ref{lem:correlation} below, we conclude the proof.
\end{proof}

\begin{customlemma}{3}\label{lem:correlation} Let $\tau_\states$ denote the sequence of states observed by PolyRL and $S_{T_e}$ be the new state visited by PolyRL. Assuming that $\tau'_\states:=(\tau_\states,S_{T_e})$ (Equation \ref{eq:omega} in the manuscript) follows the LSA-RW formalism with the persistence number $Lp_{\tau_\states}>1$, we have
 \begin{align}
 \ex{\boldsymbol{\omega}_{T_e} . \sum_{i=1}^{T_e-1} i\boldsymbol{\omega}_i}  = b^2_0 \sum_{i=1}^{T_e-1} i  e^{\frac{-|T_e-i|}{Lp_{\tau_\states}}} 
\end{align}       
\end{customlemma}

\begin{proof}\label{proof:lem-correlation}
\begin{align}
    \ex{\boldsymbol{\omega}_{T_e} . \sum_{i=1}^{T_e-1} i\boldsymbol{\omega}_i} =\ex{ \sum_{i=1}^{T_e-1} i \boldsymbol{\omega}_{T_e} . \boldsymbol{\omega}_i} = \sum_{i=1}^{T_e-1} i \ex{ \boldsymbol{\omega}_{T_e} . \boldsymbol{\omega}_i}\label{eq:proof-lem-correlation}.
\end{align}
Here, the goal is to calculate the expectation in Equation \ref{eq:proof-lem-correlation} under the assumption that $\tau'_\states$ is LSA-RW with persistence number $Lp_{\tau_\states}>1$. Note that if  $\tau'_\states$ is LSA-RW and $Lp_{\tau_\states}>1$, the chain of states visited by PolyRL prior to visiting $s_{T_e}$ is also LSA-RW with $Lp_{\tau_\states}>1$. 
Now we focus on the expectation in Equation \ref{eq:proof-lem-correlation}. We compute $\ex{\boldsymbol{\omega}_{T_e}.\boldsymbol{\omega}_i}$ using the LSA-RW formalism (Definition \ref{def:LSAW} in the manuscript) as following,
\begin{align}
    \ex{\boldsymbol{\omega}_{T_e}.\boldsymbol{\omega}_i} = b^2_0 e^{\frac{-|T_e-i|}{Lp_{\tau_\states}}}\nonumber
\end{align}
Therefore, we have, 
\begin{align}
     \sum_{i=1}^{T_e-1} i \ex{ \boldsymbol{\omega}_{T_e} . \boldsymbol{\omega}_i}&= \sum_{i=1}^{T_e-1} i b^2_0 e^{\frac{-(T_e-i)}{Lp_{\tau_\states}}}= b^2_0 \sum_{i=1}^{T_e-1} i  e^{\frac{-(T_e-i)}{Lp_{\tau_\states}}}\nonumber
\end{align}
\end{proof}
\subsection{The proof of Theorem \ref{theory:lower} in the manuscript}
\textbf{Theorem 4 statement} (Lower-Bound Theorem)%\label{theory:lower}
Let $\beta \in (0,1)$ and $\tau_\states$ be an LSA-RW in $\states$ induced by PolyRL with the persistence number $Lp_{\tau_\states}>1$ within episode $N$, $\omega_{\tau_\states}=\{\boldsymbol{\omega}_i\}_{i=1}^{T_e-1}$ be the sequence of corresponding bond vectors, where $T_e>0$ denotes the number of bond vectors within $\tau_\states$, and $b_o$ be the average bond length. The lower confidence bound for $LLS_{Ug^2}(\tau_\states)$ at least with probability $1-\delta$ is, 
\small
\begin{align}%\label{eq:lower}
    LB &= \Lambda (T_e,\tau_\states)
    +(1-\sqrt{2-2\delta})\left[\Gamma (T_e, b_o, \tau_\states) +\frac{(T_e-1)(T_e-2)}{T_e^2} b^2_0 e^{\frac{-|T_e-1|}{Lp_{\tau_\states}}}   \right],
\end{align}
where,
\begin{align}
    \Lambda (T_e,\tau_\states) = -\frac{1}{T_e-1} U_g^2(\tau_\states) \\
    \Gamma (T_e, b_o, \tau_\states) = \frac{b^2_o}{T_e}  + \frac{\norm{\sum_{i=1}^{T_e-1} i\boldsymbol{\omega}_i}^2}{T_e^3}
\end{align}

\normalsize

\begin{proof}
Using the definition of radius of gyration and letting $d=L_2$-norm in Equation \ref{eq:radius_gyration} in the manuscript, we have 
\begin{align}
     LLS_{Ug^2}(\tau_\states) &= \inf_{s_{T_e} \in \Omega} \frac{T_e-2}{T_e-1} U_g^2(\tau_\states) + \frac{1}{T_e} \norm{s_{T_e}-\Bar{\tau}_\states}^2-U_g^2(\tau_\states)\nonumber\\
    &=  -\frac{1}{T_e-1} U_g^2(\tau_\states) + \inf_{s_{T_e} \in \Omega}  \frac{1}{T_e} \norm{s_{T_e}-\Bar{\tau}_\states}^2
\end{align}
To calculate the high-probability lower bound, first we use the result from Lemma \ref{lem:bond-transform} in the manuscript. Thus, we have
\begin{align}
    \inf_{s_{T_e} \in \Omega}  \frac{1}{T_e} \norm{s_{T_e}-\Bar{\tau}_\states}^2 = \frac{1}{T_e} \inf_{s_{T_e}\in \Omega}  \norm{\boldsymbol{\omega}_{T_e} + \frac{1}{T_e} \left[ \sum_{i=1}^{T_e-1} i \boldsymbol{\omega}_i \right]}^2.
\end{align}
We subsequently use the second moment method and Paley–Zygmund inequality to calculate the high-probability lower bound. Let $Y= \norm{\boldsymbol{\omega}_{T_e} + \frac{1}{T_e} \left[ \sum_{i=1}^{T_e-1} i \boldsymbol{\omega}_i \right]}^2$, for the finite positive constants $c_1$ and $c_2$ we have,
\begin{align}\label{eq:lower-bound-main}
    \Pr[Y > c_2 \beta] \geq \frac{(1-\beta)^2}{c_1} 
\end{align}
where, 
\begin{align}\label{eq:second-moment-constraints}
    \ex{Y^2} \leq c_1 \ex{Y}^2\\ \nonumber
    \ex{Y}\geq c_2.
\end{align}
The goal is to find two constants $c_1$ and $c_2$ such that Equation \eqref{eq:second-moment-constraints} is satisfied and then we find $\beta \in (0,1)$ in Equation \eqref{eq:lower-bound-main} using $\delta$. We start by finding $c_2$ , 
\begin{align}
    \ex{Y} &= \ex{(\boldsymbol{\omega}_{T_e} + \frac{1}{T_e} \left[ \sum_{i=1}^{T_e-1} i \boldsymbol{\omega}_i \right]s).(\boldsymbol{\omega}_{T_e} + \frac{1}{T_e} \left[ \sum_{i=1}^{T_e-1} i \boldsymbol{\omega}_i \right])}\nonumber\\
    &= \ex{ \norm{\boldsymbol{\omega}_{T_e}}^2 + \frac{2}{T_e} (\boldsymbol{\omega}_{T_e}.\sum_{i=1}^{T_e-1} i\boldsymbol{\omega}_i)+\frac{1}{T_e^2}\norm{\sum_{i=1}^{T_e-1} i\boldsymbol{\omega}_i}^2}\nonumber \\
    &= \ex{\norm{\boldsymbol{\omega}_{T_e}}^2}+\frac{2}{T_e} \sum_{i=1}^{T_e-1} i \ex{\boldsymbol{\omega}_{T_e}.\boldsymbol{\omega}_i} + \frac{1}{T_e^2}\norm{\sum_{i=1}^{T_e-1} i\boldsymbol{\omega}_i}^2\nonumber\\
    &= b^2_o +\frac{2}{T_e} b^2_0 \sum_{i=1}^{T_e-1} i  e^{\frac{-|T_e-i|}{Lp_{\tau_\states}}} + \frac{1}{T_e^2}\norm{\sum_{i=1}^{T_e-1} i\boldsymbol{\omega}_i}^2,\label{eq:E-y}
\end{align}
therefore, 
\begin{align}
    \ex{Y} &= b^2_o +\frac{2}{T_e} b^2_0 \sum_{i=1}^{T_e-1} i  e^{\frac{-|T_e-i|}{Lp_{\tau_\states}}} + \frac{1}{T_e^2}\norm{\sum_{i=1}^{T_e-1} i\boldsymbol{\omega}_i}^2\nonumber\\
    &\geq b^2_o +\frac{2}{T_e} b^2_0 e^{\frac{-|T_e-1|}{Lp_{\tau_\states}}} \sum_{i=1}^{T_e-1} i  + \frac{1}{T_e^2}\norm{\sum_{i=1}^{T_e-1} i\boldsymbol{\omega}_i}^2 \nonumber\\
    = & \underbrace{b^2_o +\frac{(T_e-1)(T_e-2)}{T_e} b^2_0 e^{\frac{-|T_e-1|}{Lp_{\tau_\states}}}  + \frac{1}{T_e^2}\norm{\sum_{i=1}^{T_e-1} i\boldsymbol{\omega}_i}^2 }_{=c_2}
\end{align}

To find $c_1$, we have
\begin{align}
    \ex{Y^2} \leq c_1 \ex{Y}^2.\nonumber
\end{align}

\begin{align}
    \ex{Y^2} &= \ex{ \left(\norm{\boldsymbol{\omega}_{T_e}}^2 + \frac{2}{T_e} (\boldsymbol{\omega}_{T_e}.\sum_{i=1}^{T_e-1} i\boldsymbol{\omega}_i)+\frac{1}{T_e^2}\norm{\sum_{i=1}^{T_e-1} i\boldsymbol{\omega}_i}^2 \right)^2}\nonumber\\
    &= \ex{\norm{\boldsymbol{\omega}_{T_e}}^4} +\frac{4}{T_e^2} \ex{ (\boldsymbol{\omega}_{T_e}.\sum_{i=1}^{T_e-1} i\boldsymbol{\omega}_i)^2} \nonumber\\
    &+\frac{1}{T_e^4} \ex{\norm{\sum_{i=1}^{T_e-1} i\boldsymbol{\omega}_i}^4} + \frac{4}{T_e}\ex{ \norm{\boldsymbol{\omega}_{T_e}}^2 (\boldsymbol{\omega}_{T_e}.\sum_{i=1}^{T_e-1} i\boldsymbol{\omega}_i)} \nonumber\\
    &+\frac{2}{T_e^2} \ex{ \norm{\boldsymbol{\omega}_{T_e}}^2 \norm{\sum_{i=1}^{T_e-1} i\boldsymbol{\omega}_i}^2} +\frac{2}{T_e^3} \ex{(\boldsymbol{\omega}_{T_e}.\sum_{i=1}^{T_e-1} i\boldsymbol{\omega}_i)\norm{\sum_{i=1}^{T_e-1} i\boldsymbol{\omega}_i}^2 }\nonumber\\
    &= \ex{\norm{\boldsymbol{\omega}_{T_e}}^4} +\frac{4}{T_e^2} \ex{ (\boldsymbol{\omega}_{T_e}.\sum_{i=1}^{T_e-1} i\boldsymbol{\omega}_i)^2}\nonumber \\
    &+\frac{1}{T_e^4} \norm{\sum_{i=1}^{T_e-1} i\boldsymbol{\omega}_i}^4 + \frac{4}{T_e}\ex{ \norm{\boldsymbol{\omega}_{T_e}}^2 (\boldsymbol{\omega}_{T_e}.\sum_{i=1}^{T_e-1} i\boldsymbol{\omega}_i)} \nonumber\\
    &+\frac{2b^2_o}{T_e^2} \norm{\sum_{i=1}^{T_e-1} i\boldsymbol{\omega}_i}^2  +\frac{2}{T_e^3} \norm{\sum_{i=1}^{T_e-1} i\boldsymbol{\omega}_i}^2 \ex{(\boldsymbol{\omega}_{T_e}.\sum_{i=1}^{T_e-1} i\boldsymbol{\omega}_i) }\label{eq:E-Y2}
\end{align}
We calculate the expectations appearing in Equation \eqref{eq:E-Y2} to conclude the proof.

\begin{align}
    \frac{4}{T_e^2} \ex{ (\boldsymbol{\omega}_{T_e}.\sum_{i=1}^{T_e-1} i\boldsymbol{\omega}_i)^2} &\leq \frac{4}{T_e^2}  \ex{ (\norm{\boldsymbol{\omega}_{T_e}}\norm{\sum_{i=1}^{T_e-1} i\boldsymbol{\omega}_i})^2}\nonumber\\
    &= \frac{4\norm{\sum_{i=1}^{T_e-1} i\boldsymbol{\omega}_i}^2}{T_e^2}  \ex{ (\norm{\boldsymbol{\omega}_{T_e}})^2} \nonumber\\
    &= \frac{4\norm{\sum_{i=1}^{T_e-1} i\boldsymbol{\omega}_i}^2 b^2_o}{T_e^2}\label{eq:expect-2}
\end{align}
\begin{align}
    \frac{4}{T_e}\ex{ \norm{\boldsymbol{\omega}_{T_e}}^2 (\boldsymbol{\omega}_{T_e}.\sum_{i=1}^{T_e-1} i\boldsymbol{\omega}_i)}
    &=\frac{4}{T_e} \ex{\sum_{i=1}^{T_e-1} i \norm{\boldsymbol{\omega}_{T_e}}^2 \boldsymbol{\omega}_{T_e}.\boldsymbol{\omega}_i }\nonumber\\
    &=\frac{4}{T_e} \sum_{i=1}^{T_e-1} i \ex{\norm{\boldsymbol{\omega}_{T_e}}^2 \boldsymbol{\omega}_{T_e}.\boldsymbol{\omega}_i }\nonumber\\
    &=\frac{4 \max_{s,s'} \norm{\boldsymbol{\omega}(s,s')}^2}{T_e} \sum_{i=1}^{T_e-1} i \ex{ \boldsymbol{\omega}_{T_e}.\boldsymbol{\omega}_i }\nonumber\\
    &=\frac{4b^2_o \max_{s,s'} \norm{\boldsymbol{\omega}(s,s')}^2}{T_e} \sum_{i=1}^{T_e-1} i e^{\frac{-(T_e-i)}{Lp_{\tau_\states}}}\label{eq:expect-3}
\end{align}
where $\boldsymbol{\omega}(s,s')$ denotes the bond vector between two states $s$ and $s'$.

To calculate $\ex{\norm{\boldsymbol{\omega}_{T_e}}^4}$, we let $Z \sim \mathcal{N}(0,1)$ and using definition 1, w.lo.g. we assume $\norm{\boldsymbol{\omega}_{T_e}}^2 \sim \mathcal{N}(b^2_o,\sigma^2)$ with $\sigma < \infty$. Thus, we have
\begin{align}
    &\ex{\norm{\boldsymbol{\omega}_{T_e}}^4} = \ex{\sigma^2 Z^2+2b^2_o\sigma Z + b^2_o} \nonumber\\
    &\underbrace{=}_{\text{Binomial Theorem and linearity of expectation}} \sigma^2 +b^4_o \label{eq:expect-1}
\end{align}

\begin{align}
     \ex{(\boldsymbol{\omega}_{T_e}.\sum_{i=1}^{T_e-1} i\boldsymbol{\omega}_i)} = \ex{(\sum_{i=1}^{T_e-1} i\boldsymbol{\omega}_{T_e}.\boldsymbol{\omega}_i)} = \sum_{i=1}^{T_e-1}i\ex{\boldsymbol{\omega}_{T_e}.\boldsymbol{\omega}_i}= \sum_{i=1}^{T_e-1}i b^2_o e^{\frac{-|T_e-i|}{Lp_{\tau_\states}}}\label{eq:expect-4}
\end{align}

Substitution of the expectations in Equation \eqref{eq:E-Y2} with Equations \eqref{eq:expect-1}, \eqref{eq:expect-2}, \eqref{eq:expect-3} and \eqref{eq:expect-4} gives,
\begin{align}
    \ex{Y^2} &\leq \sigma^2 +b^4_o  + \frac{4b^2_o}{T_e^2}\norm{\sum_{i=1}^{T_e-1} i\boldsymbol{\omega}_i}^2 + \frac{1}{T_e^4} \norm{\sum_{i=1}^{T_e-1} i\boldsymbol{\omega}_i}^4+ \frac{4b^2_o \max_{s,s'} \norm{\boldsymbol{\omega}(s,s')}^2}{T_e} \sum_{i=1}^{T_e-1} ie^{\frac{-(T_e-i)}{Lp_{\tau_\states}}}\nonumber\\ 
    &+ \frac{2b^2_o}{T_e^2} \norm{\sum_{i=1}^{T_e-1} i\boldsymbol{\omega}_i}^2 + \frac{2b^2_o}{T_e^3} \sum_{i=1}^{T_e-1}i  e^{\frac{-|T_e-i|}{Lp_{\tau_\states}}} \norm{\sum_{i=1}^{T_e-1} i\boldsymbol{\omega}_i}^2 
\end{align}
Equation \eqref{eq:E-y} gives,
\begin{align}\label{eq:denominator}
    \ex{Y}^2 &= b^4_o+ \frac{4b^4_o}{T_e^2}  \left(\sum_{i=1}^{T_e-1} i  e^{\frac{-(T_e-i)}{Lp_{\tau_\states}}}\right)^2 + \frac{1}{T_e^4}\norm{\sum_{i=1}^{T_e-1} i\boldsymbol{\omega}_i}^4+  \frac{2b^4_o}{T_e^2} \sum_{i=1}^{T_e-1} i  e^{\frac{-(T_e-i)}{Lp_{\tau_\states}}}\nonumber\\
    &+\frac{2b^2_o}{T_e^2}\norm{\sum_{i=1}^{T_e-1} i\boldsymbol{\omega}_i}^2  + \frac{2b^2_o}{T_e^3} \sum_{i=1}^{T_e-1} i  e^{\frac{-(T_e-i)|}{Lp_{\tau_\states}}} \norm{\sum_{i=1}^{T_e-1} i\boldsymbol{\omega}_i}^2
\end{align}

Now to find $c_1$, we use Equation \eqref{eq:second-moment-constraints},

\begin{align}
    \frac{4b^2_o}{T_e^2}\norm{\sum_{i=1}^{T_e-1} i\boldsymbol{\omega}_i}^2 - \frac{4b^4_o}{T_e^2}  \left(\sum_{i=1}^{T_e-1} i  e^{\frac{-|T_e-i|}{Lp_{\tau_\states}}}\right)^2 &= \frac{4b^2_o}{T_e^2}\left(\norm{\sum_{i=1}^{T_e-1} i\boldsymbol{\omega}_i}^2 - \left(\sum_{i=1}^{T_e-1} i  e^{\frac{-|T_e-i|}{Lp_{\tau_\states}}}\right)^2 \right)\nonumber\\
    &\leq \underbrace{\frac{4b^2_o}{T_e^2}\left(\norm{\sum_{i=1}^{T_e-1} i\boldsymbol{\omega}_i}^2\right)}_{B}.\nonumber
\end{align}
\begin{align}
    &\frac{4b^2_o \max_{s,s'} \norm{\boldsymbol{\omega}(s,s')}^2}{T_e} \sum_{i=1}^{T_e-1} i e^{\frac{-(T_e-i)}{Lp_{\tau_\states}}} - \frac{2b^4_o}{T_e^2} \sum_{i=1}^{T_e-1} i  e^{\frac{-(T_e-i)}{Lp_{\tau_\states}}} \nonumber\\
     &=\left( \frac{4b^2_o \max_{s,s'} \norm{\boldsymbol{\omega}(s,s')}^2}{T_e}  -  \frac{2b^4_o}{T_e^2} \right)\sum_{i=1}^{T_e-1} i e^{\frac{-(T_e-i)}{Lp_{\tau_\states}}} \nonumber\\
     &\leq \underbrace{\left( \frac{4b^2_o \max_{s,s'} \norm{\boldsymbol{\omega}(s,s')}^2}{T_e} \right) \sum_{i=1}^{T_e-1} i e^{\frac{-(T_e-i)}{Lp_{\tau_\states}}}}_{A}\nonumber
\end{align}

Thus, we have
\begin{align}
    \frac{\ex{Y^2}}{\ex{Y}^2} \leq 1 + \frac{\sigma^2 + A+ B}{\ex{Y}^2} \underbrace{\leq }_{\text{by comparing A and B with \eqref{eq:denominator}}} 2= c_1
\end{align}
\end{proof}

\subsection{The proof of Corollary \ref{cor:PolyRL-LSAW} in the manuscript} 

\textbf{Corollary 5 statement:} Given that assumption $1$ is satisfied, any exploratory trajectory induced by PolyRL algorithm (ref. Algorithm \ref{alg:plyrl} in the manuscript) with high probability is an LSA-RWs.

\begin{proof}
Given Assumption \ref{as:1} in the manuscript, due to the Lipschitzness of the transition probability kernel w.r.t. the action variable, the change in the distributions of the resulting states are finite and bounded by the $L_2$ distance of the actions. Thus, given a locally self-avoiding chain $\tau_{\actions} \in \actions^{T_e}$ with persistence number $Lp_{\tau_{\actions}}$, and $\forall i \in [T_e]:  b^2_o = \mathbb{E}[\norm{a_i}^2]$, by the Lipschitzness of the transition probability kernel of the underlying MDP, there exists a finite empirical average bond vector among the states visited by PolyRL (\emph{i.e.} the first condition in Definition \ref{def:LSAW} in the manuscript is satisfied). 
 
On the other hand, the PolyRL action sampling method (Algorithm 2) by construction preserves the expected correlation angle $\theta_{\tau_{\actions}}$ between the consecutive selected actions with finite $L_2$ norm, leading to a locally self-avoiding random walk in $\actions$. Given the following measure of spread adopted by PolyRL and defined as, 
\begin{align}
    U_g^2(\tau_\states) &:= \frac{1}{T_e-1} \sum_{s \in \tau_\states} \norm{s-\Bar{\tau}_\states}^2,%\label{eq:radius_gyration}
\end{align}
and the results of Theorems \ref{theory:upper} and \ref{theory:lower} in the manuscript ($LB$ and $UB$ high probability confidence bounds on the sensitivity of $U_g^2(.)$), and considering that at each time step the persistence number of the chain of visited states $Lp_{\tau_\states}$ is calculated and the exploratory action is selected such that the stiffness of $\tau_\states$ is preserved, with probability $1-\delta$ the correlation between the bonds in $\tau_\states$ is maintained (\emph{i.e.} the second condition in Definition \ref{def:LSAW} in the manuscript is satisfied). Hence, with probability $1-\delta$ the chain $\tau_{\states}$ induced by PolyRL is locally self avoiding.
\end{proof}

\begin{corollary}
Under assumption \ref{as:1} in the manuscript, with high probability the $T_e$ time-step exploratory chain $\tau$ induced by PolyRL with persistence number $Lp_\tau$ provides higher space coverage compared with the $T_e$ time-step exploratory chain $\tau'$ generated by a random-walk model.% up to the factor $ \frac{1+e^{-1/Lp_\tau}}{1-e^{-1/Lp_\tau}}$.
\end{corollary}
\begin{proof}
Results from Corollary \ref{cor:PolyRL-LSAW} in the manuscript together with Remark \ref{rem:1}  conclude the proof. 
\end{proof}

\section{Action Sampling Method}\label{sec:action-sampling}
In this section, we provide the action sampling algorithm (Algorithm \ref{alg:actionSampling}), which contains the step-by-step instruction for sampling the next action. The action sampling process is also graphically presented in Figure \ref{fig:state-action-traj} (a).
\begin{figure*}[ht]
\begin{center}
\includegraphics[width=1 \textwidth]{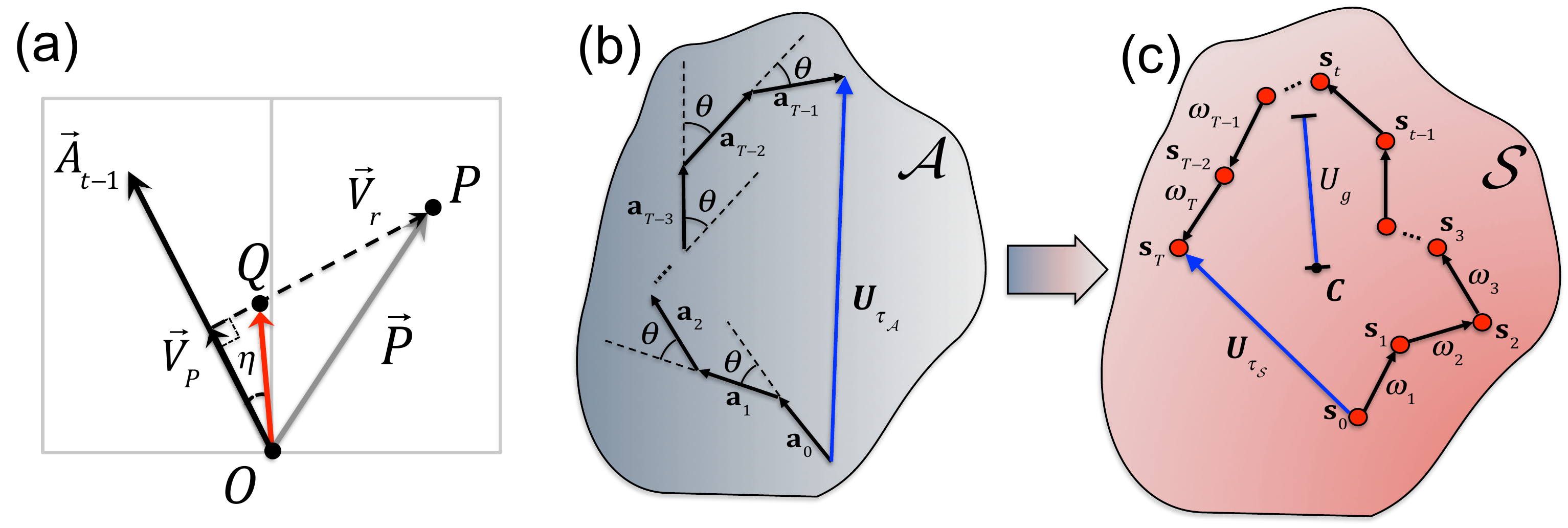}
\end{center}
\caption{Schematics of the steps involved in the PolyRL exploration technique. (a) The action sampling method. In order to choose the next action $\Vec{\bf A}_t$, a randomly chosen point $P$ in $\actions$ is projected onto the current action vector $\Vec{\bf A}_{t-1}$, which gives $\Vec{\bf V}_P$. The point $Q$ is subsequently found on the vector $\Vec{\bf V}_r = \vec{\bf P}-\vec{\bf V}_P$ using trigonometric relations and the angle $\eta$ drawn from a normal distribution with mean $\theta$. The resulting vector $\vec{OQ}$ (shown in red) gives the next action. Detailed instructions are given in Algorithm 2. (b) A schematic of action trajectory $\tau_\actions$ with the mean correlation angle $\theta$ between every two consecutive bond vectors and the end-to-end vector ${\bf U}_{\tau_\actions}$. (c) A schematic of state trajectory $\tau_\states$ with bond vectors $ \boldsymbol{\omega}_i={\bf s}_i-{\bf s}_{i-1}$. The radius of gyration and the end-to-end vector are depicted as $U_g$ and ${\bf U}_{\tau_\states}$, respectively. Point ${\bf C}$ is the center of mass of the visited states.
}\label{fig:state-action-traj}
\end{figure*}
\setcounter{algorithm}{1}
\begin{algorithm}[H]
\caption{Action Sampling}\label{alg:actionSampling}
\begin{algorithmic}[1]
\REQUIRE Angle $\eta$ and Previous action ${\bf A}_{t-1}$
\STATE Draw a random point $P$ in the action space ($P_i\sim\mathcal{U}[-m,m]; i=1,\dots d$)   \COMMENT{${\bf P}$ is the vector from the origin to the point $P$}
\STATE $D={\bf A}_{t-1}.{\bf P}$
\STATE ${\bf V}p=\dfrac{D}{||{\bf A}_{t-1}||_2^2}{\bf A}_{t-1}$ \COMMENT{The projection of ${\bf P}$ on ${\bf A}_{t-1}$}
\STATE ${\bf V}r={\bf P}-{\bf V}p$
\STATE $l=||{\bf V}p||_2\tan\eta$
\STATE $k=l/||{\bf V}r||_2$
\STATE ${\bf Q}=k{\bf V}r+{\bf V}p$
\IF {$D>0$}
	\STATE ${\bf A}_t={\bf Q}$
\ELSE
	\STATE ${\bf A}_t=-{\bf Q}$
\ENDIF
\STATE Clip ${\bf A}_t$ if out of action range
\STATE Return ${\bf A}_t$
\end{algorithmic}
\end{algorithm}

\section{Additional Baseline Results} \label{sec:aditional-baseline}
In this section, we provide the benchmarking results for DDPG-UC, DDPG-OU, DDPG-PARAM, DDPG-FiGAR (Figure \ref{fig:Dense}), as well as SAC and OAC (Figure \ref{fig:Dense_SAC}) algorithms on three standard MuJoco tasks. Moreover, the source code is provided \href{https://github.com/h-aboutalebi/SparseBaseline.git}{here}.
\begin{figure}[ht]
\begin{center}
\includegraphics[width=1 \textwidth]{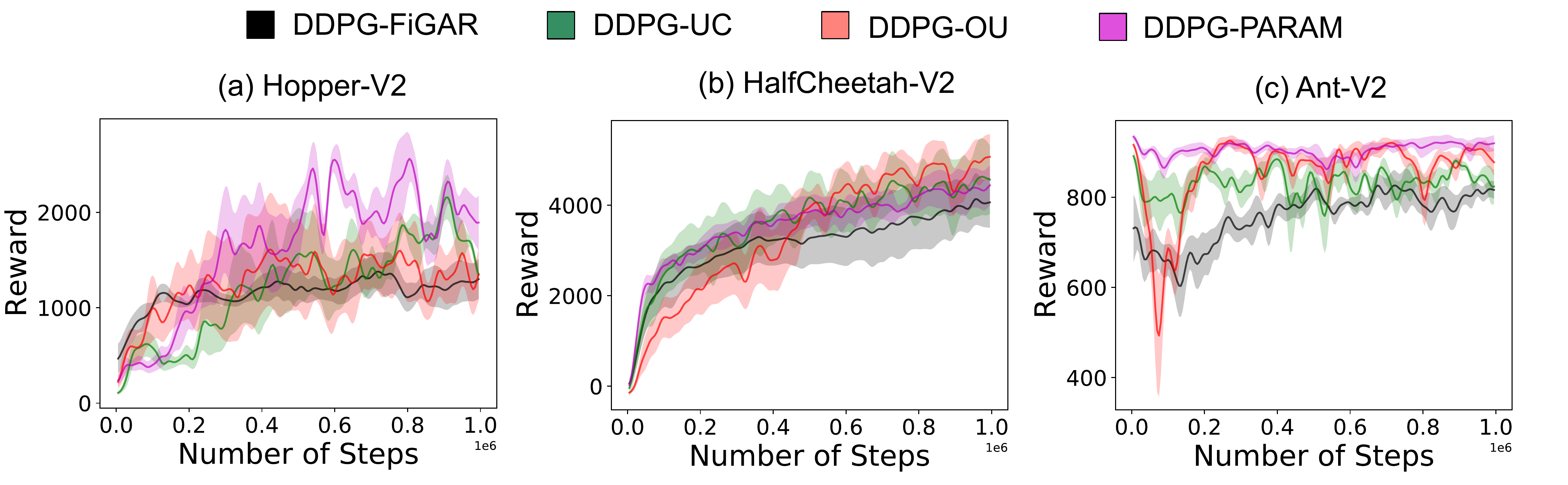}
\end{center}
\caption{Performance of DDPG-UC, DDPG-OU, DDPG-PARAM, and DDPG-FiGAR algorithms across 3 MuJoCo domains. The plots are averaged over 5 random seeds. The test evaluation happens every 5k over 1 million time steps.}
\label{fig:Dense}
\end{figure}

\begin{figure}[H]
\begin{center}
\includegraphics[width=1 \textwidth]{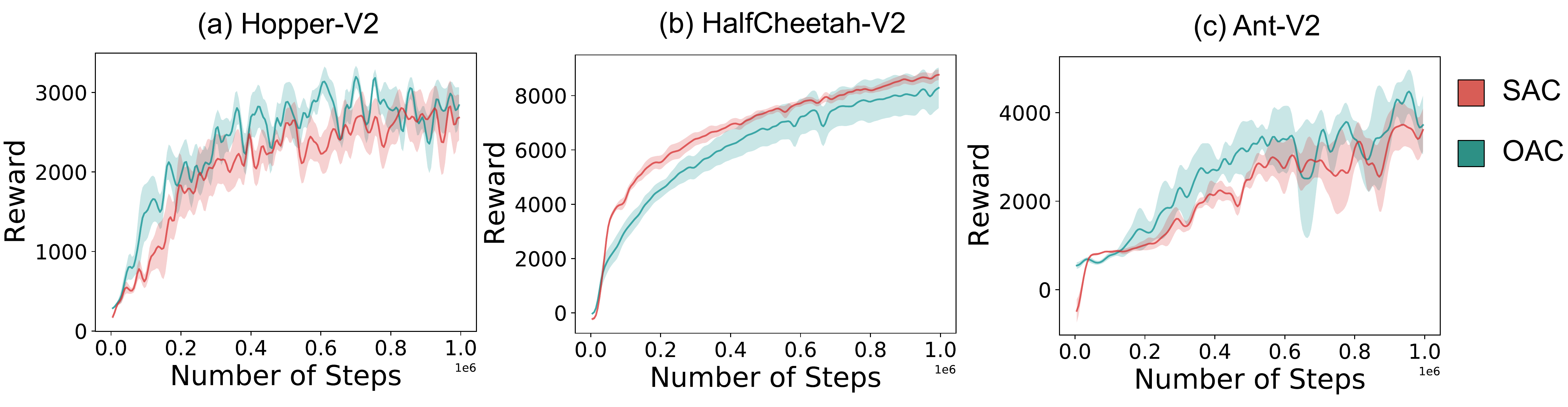}
\end{center}
\caption{Performance of SAC and OAC algorithms across 3 MuJoCo domains. The plots are averaged over 5 random seeds. The test evaluation happens every 5k steps over 1 million time steps.}
\label{fig:Dense_SAC}
\end{figure}
In order to Benchmark DDPG-FiGAR results, we let the action repetition set, defined as $W := \{1,2,\dots |W|\}$ (\cite{sharma2017learning}), be equal to $\{1\}$. The results are expected to converge to those of DDPG-OU noise as depicted in Figure \ref{fig:FIGAR-Bench}.
\begin{figure}[H]
\begin{center}
\includegraphics[width=1 \textwidth]{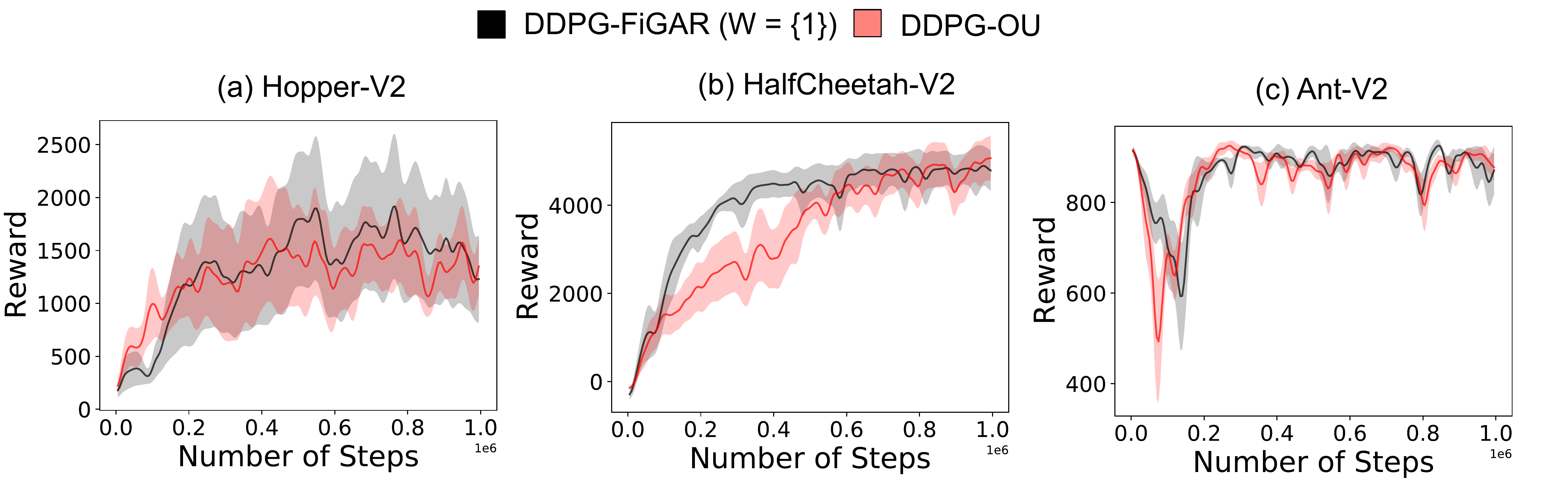}
\end{center}
\caption{Benchmarking DDPG-FiGAR against DDPG-OU using action repetition set $W = \{1\}$ across 3 MuJoCo domains. The plots are averaged over 5 random seeds. The test evaluation happens every 5k steps over 1 million time steps.}
\label{fig:FIGAR-Bench}
\end{figure}

\section{Hyperparameters and Network Architecture}\label{sec:hyperparams}
In this section, we provide the architecture of the neural networks (Table \ref{tab:Network}), as well as the PolyRL hyper parameters (Table \ref{tab:PolyRL}) used in the experiments. Regarding the computing infrastructure, the experiments were run on a slurm-managed cluster with NVIDIA P100 Pascal (12G HBM2 memory) GPUs. The avergae run-time for DDPG-based and SAC-based models were around $8$ and $12$ hours, respectively. 

\begin{table}[H]
\caption{DDPG and SAC Network Architecture}
\centering
    \begin{tabular}{|p{0.5\textwidth}|p{0.5\textwidth}|}
    \hline
     \textbf{Parameter}  &  \textbf{Value}\\ \hline\hline
       \textsl{Optimizer}  &  Adam\\ \hline
       \textsl{Critic Learning Rate} & \begin{tabular}{@{}p{0.25\textwidth}||p{0.25\textwidth}@{}}
                   $1\mathrm{e}{-3}$ (DDPG) & $3\mathrm{e}{-4}$ (SAC) \\
                 \end{tabular}  \\ \hline
    \textsl{Actor Learning Rate} & \begin{tabular}{@{}p{0.25\textwidth}||p{0.25\textwidth}@{}}
                   $1\mathrm{e}{-4}$ (DDPG) & $3\mathrm{e}{-4}$ (SAC) \\
                 \end{tabular}  \\ \hline
    \textsl{Discount Factor} & 0.99 \\ \hline
    \textsl{Replay Buffer Size} & $1\mathrm{e}{+6}$\\ \hline
    \textsl{Number of Hidden Layers (All Networks)} & $2$\\ \hline
    \textsl{Number of Units per Layer} & \begin{tabular}{@{}p{0.25\textwidth}||p{0.25\textwidth}@{}}
                   $400$ ($1^{\mbox{\scriptsize st}}$)- $300$ ($2^{\mbox{\scriptsize nd}}$) (DDPG) & both $256$ (SAC) \\
                 \end{tabular}  \\ \hline
    \textsl{Number of Samples per Mini Batch} & $100$\\ \hline
    \textsl{Nonlinearity} & ReLU\\ \hline
    \textsl{Target Network Update Coefficient} & $5\mathrm{e}{-3}$\\ \hline
    \textsl{Target Update Interval} & $1$\\ \hline
    \end{tabular}
    \label{tab:Network}
\end{table}

\textbf{The exploration factor -} The one important parameter in the PolyRL exploration method, which controls the exploration-exploitation trade-off is the exploration factor $\beta\in [0,1]$. The factor $\beta$ plays the balancing role in two ways: controlling (1) the range of confidence interval (Equations \ref{eq:upper} and \ref{eq:lower} in the manuscript; $\delta = 1-e^{-\beta N}$); and (2) the probability of switching from the target policy $\pi_\mu$ to the behaviour policy $\pi_{\mbox{\scriptsize PolyRL}}$. Figure \ref{fig:heat_map} illustrates the effect of varying $\beta$ on the performance of a DDPG-PolyRL agent in the HalfCheetah-v2 environment. The heat maps (Figures \ref{fig:heat_map} (a), (b) and (c)) show the average asymptotic reward obtained for different pairs of correlation angle $\theta$ and variance $\sigma^2$. The heat maps depict that for this specific task, the performance of DDPG-PolyRL improves as $\beta$ changes from $0.0004$ to $0.01$. The performance plot for the same task (Figure \ref{fig:heat_map} (d)) shows the effect of $\beta$ on the amount of the obtained reward. The relation of $\beta$ with the percentage of the moves taken using the target policy is illustrated in Figure \ref{fig:heat_map} (e)). As expected, larger values of $\beta$ lead to more exploitation and fewer exploratory steps.
\begin{figure*}[ht]
\begin{center}
\includegraphics[width=1\textwidth]{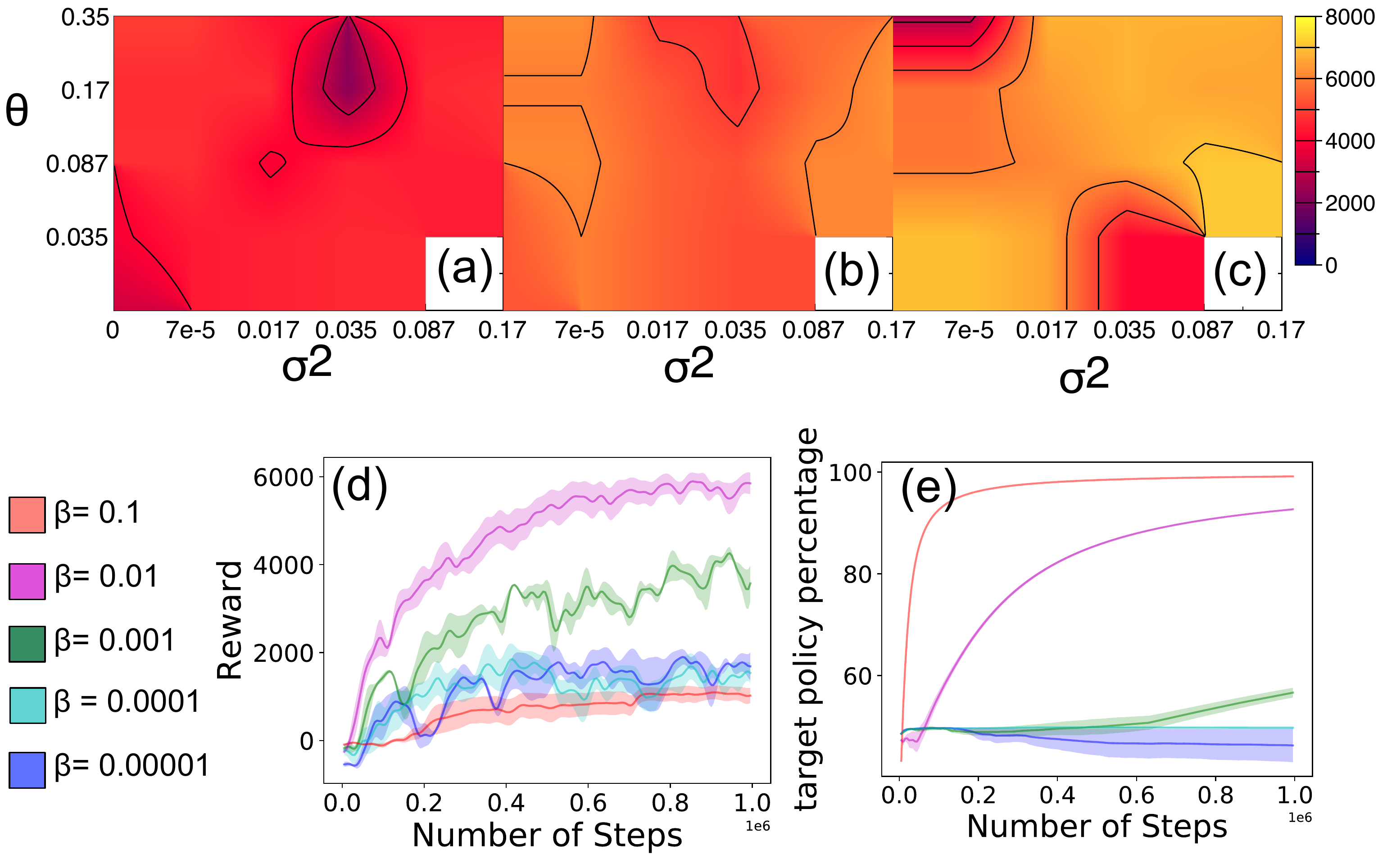}
\end{center}
\caption{Performance of DDPG-PolyRL in HalfCheetah-v2 for different values of exploration factor $\beta$. (a-c) Heat maps depict the mean of the obtained asymptotic rewards after 3 million time steps over a range of correlation angle $\theta$ and the variance $\sigma^2$. The results are shown for $\beta = 0.0004$ (a), $\beta = 0.001$ (b), and $\beta = 0.01$ (c). (d) Performance of DDPG-PolyRL in HalfCheetah-v2 for the fixed values of $\theta = 0.035$ and $\sigma^2 = 0.00007$, and different values of $\beta$. (e) The percentage of the movements the DDPG-PolyRL agent behaves greedily. All values are averaged over four random seeds and the error bars show the standard error on the mean.}
\label{fig:heat_map}
\end{figure*}

\begin{table}[H]
    \centering   
    \caption{PolyRL Hyper parameters. Note that the parameters $\theta$ and $\sigma$ are angles and their respective values in the table are in radian.}
    \begin{tabular}{|c|c|c|c|}
    \hline
    \multicolumn{1}{|c|}{\cellcolor[HTML]{D0D0D0}} & \begin{tabular}{@{}c@{}}
    \textbf{Mean Correlation Angle}\\
    $\boldsymbol{\theta}$\\
    \end{tabular} & \begin{tabular}{@{}c@{}}
    \textbf{Variance} \\
    $\boldsymbol{\sigma^2}$\\
    \end{tabular} & \begin{tabular}{@{}c@{}}
    \textbf{Exploration Factor} \\
    $\boldsymbol{\beta}$\\
    \end{tabular}\\
    \hline\hline
    \multicolumn{4}{|c|}{\emph{DDPG-PolyRL}}\\ \hline
    \textsl{SparseHopper-V2 ($\lambda = 0.1$)} & 0.035 & 0.00007 & 0.001\\ \hline
    \textsl{SparseHalfCheetah-V2 ($\lambda = 5$)} & 0.17 & 0.017 & 0.02 \\ \hline
    \textsl{SparseAnt-V2 ($\lambda = 0.15$)} & 0.087 & 0.035 & 0.01\\ \hline\hline
    \multicolumn{4}{|c|}{\emph{SAC-PolyRL}}\\ \hline
    \textsl{SparseHopper-V2 ($\lambda = 3$)} & 0.35 & 0.017 & 0.01 \\ \hline
    \textsl{SparseHalfCheetah-V2 ($\lambda = 15$)} & 0.35 & 0.00007 & 0.05 \\ \hline
    \textsl{SparseAnt-V2 ($\lambda = 3$)} & 0.035 & 0.00007 & 0.01 \\ \hline
    \end{tabular}
    \label{tab:PolyRL}
\end{table}

\clearpage
\end{document}

%% file: macros.tex
\newcommand{\norm}[1]{\| #1 \|}

\newcommand{\ex}[1]{\mathbb{E}\left[ #1 \right]}

\newcommand{\states}{\mathcal{S}}
\newcommand{\actions}{\mathcal{A}}

\newtheorem{theorem}{Theorem}
\newtheorem{assumption}[theorem]{Assumption}

\newtheorem{lemma}[theorem]{Lemma}

\newtheorem{corollary}[theorem]{Corollary}

\newtheorem{remark}{Remark}
\newtheorem{definition}{Definition}